\documentclass[
aps,
prx,
reprint,
superscriptaddress,
longbibliography,
]{revtex4-2}

\usepackage{amsmath,amssymb,amsfonts,amsthm}
\usepackage{graphicx}
\usepackage{textcomp}
\usepackage[dvipsnames]{xcolor}
\usepackage{siunitx}
\usepackage{amsmath,amsfonts,bm}

\def\eqref#1{equation~\ref{#1}}
\def\1{\bm{1}}

\DeclareMathAlphabet{\mathsfit}{\encodingdefault}{\sfdefault}{m}{sl}
\SetMathAlphabet{\mathsfit}{bold}{\encodingdefault}{\sfdefault}{bx}{n}

\usepackage{comment}
\usepackage{wrapfig}
\usepackage{url}
\usepackage{booktabs}
\usepackage{multirow}
\usepackage{tabularx}
\usepackage{minted}
\usepackage{cancel}
\usepackage{algorithm}
\usepackage{algpseudocode}
\usepackage[most]{tcolorbox}

\usepackage{hyperref}
\hypersetup{
  colorlinks = true,
  urlcolor   = Blue,
  linkcolor  = Blue,
  citecolor  = Blue
}

\definecolor{mintgreen}{HTML}{ABE4D7}
\definecolor{tealgreen}{HTML}{D2E8E4}

\newtheorem{theorem}{Theorem}[section]

\newtheorem{proposition}[theorem]{Proposition}

\newtheorem{definition}[theorem]{Definition}

\newtheorem{assumption}[theorem]{Assumption}

\begin{document}

\title{Toward Generative Quantum Utility via Correlation–Complexity Map}

\author{Chen-Yu Liu}
\email{chen-yu.liu@quantinuum.com}
\affiliation{Quantinuum, Partnership House, London, UK}

\author{Leonardo Placidi}
\email{leonardo.placidi@quantinuum.com}
\affiliation{Quantinuum, Otemachi Financial City Grand Cube, Tokyo, Japan}

\author{Eric Brunner}
\email{eric.brunner@quantinuum.com}
\affiliation{Quantinuum, Partnership House, London, UK}

\author{Enrico Rinaldi}
\email{enrico.rinaldi@quantinuum.com}
\affiliation{Quantinuum, Partnership House, London, UK}

\newcommand{\fix}{\marginpar{FIX}}
\newcommand{\new}{\marginpar{NEW}}

\begin{abstract}
We study a practical question in generative quantum machine learning: given a classical dataset, can we determine, before training, whether it is well suited to a quantum generative model? We focus on a class of quantum circuits known as instantaneous quantum polynomial-time (IQP) circuits, whose output distributions are widely believed to be difficult to sample from using classical methods. These circuits are used to build our quantum generative models. We introduce a Correlation–Complexity Map, a simple diagnostic built from two quantities computed from data samples. The first measures how closely the dataset’s spectral correlation patterns resemble those naturally produced by IQP circuits, while the second quantifies how much of the dataset’s structural correlation cannot be captured by simple pairwise models. 
In other words, we can estimate beforehand how well a dataset can be approximated by our model family and also how complex its correlations are, indicating possible failures of classical models.
Applying this framework, we identify turbulence data as a promising target for quantum generative modeling. Guided by this analysis, we use a latent-parameter adaptation scheme that reuses a compact IQP circuit over a temporal sequence by learning and interpolating a low-dimensional latent trajectory, and observe competitive performance against classical baselines in a low-data, low-parameter regime. These results suggest that dataset-level diagnostics can help prioritize problems where quantum generative models are most likely to be useful, with improvements in data and parameter efficiency.
\end{abstract}

\maketitle

\section{Introduction}
Quantum computing \citep{nielsen2010quantum} promises a new computational paradigm, with the potential to represent probability distributions that are difficult to reproduce using classical methods \citep{gharibyan2025heuristic, arute2019quantum}. In parallel, quantum machine learning (QML) \citep{schuld2019quantum, benedetti2019parameterized, mari2020transfer, liu2025quantum, huang2021power, belis2026spectral} has emerged as a framework for leveraging quantum effects to construct models for data analysis and generation. Among the various QML approaches \citep{dong2008quantum, chen2020variational, liu2024training, lin2024quantum, hanruiwang2022quantumnas, du2021grover}, generative modeling \citep{zhang2024generative, lloyd2018quantum, huang2025generative, recio2025train} is particularly promising, since sampling is a native operation of quantum devices and may offer practical utility.

A widely studied class of quantum generative models is based on instantaneous quantum polynomial-time (IQP) circuits, whose output distributions are believed to be hard to sample from classically. Indeed, efficient classical sampling from IQP circuits under standard assumptions would imply a collapse of the polynomial hierarchy (PH) \citep{bremner2011classical}, providing a strong complexity-theoretic argument that IQP sampling is classically intractable at scale. This perspective has motivated efforts to identify settings in which IQP-based models may offer a generative quantum advantage \citep{huang2025generative, recio2025train, ballo2026shallow}.

A central obstacle, however, is their trainability at scale. Recent work proposed a scalable workflow, ``train on classical, deploy on quantum'', that leverages properties of IQP models to perform optimization on classical hardware while reserving quantum hardware for the sampling stage \citep{recio2025train}. 
In this approach, training objectives can be computed using classically tractable quantities (e.g., expectation values of commuting observables), avoiding the cost and instability of hardware-in-the-loop training while preserving the possibility that classical simulation of the sampling process is hard as system size increases. Related approaches based on local inversion and sewing techniques have also been proposed to improve scalability, although they impose additional structural constraints on the circuit \citep{huang2025generative}.

Despite this progress, two open directions remain particularly pressing for useful generative quantum advantage: (i) moving beyond purely bitstring-native datasets toward regimes that naturally support floating-point and integer-valued data, and (ii) improving the search for practical data distributions whose structure matches the inductive biases of classically hard quantum models. These issues shift the emphasis from proving worst-case hardness to diagnosing when a given real-world dataset is structurally compatible with an IQP-type generative model, and doing so in a way that can guide dataset selection and model design, {before investing significant research resources \citep{huang2021power}.}
The search for \emph{problem instances} where quantum algorithms can provide advantage was recently studied and formalized in~\cite{buhrman2025formal}.

In this work, we propose a Correlation–Complexity Map designed to operationalize this search. The map is defined by two complementary dataset indicators. First, we introduce a spectral measure, the Quantum Correlation–Likeness Indicator (QCLI) $I_\text{QCLI}$, that quantifies how the dataset’s correlation-order (Walsh–Fourier) spectrum deviates from an i.i.d. classical binomial baseline via a Jensen–Shannon divergence.
Second, we introduce a structural measure, the classical Correlation–Complexity Indicator (CCI) $I_\text{CCI}$, which measures the fraction of total correlation not captured by an optimal tree-structured model, computed via the Chow–Liu tree approximation. Together, the pair $(I_\text{QCLI}, I_\text{CCI})$ yields a two-dimensional landscape that separates datasets that are classically local/pairwise-expressible from those dominated by genuinely high-order dependencies, while simultaneously measuring whether their correlation signatures are compatible with IQP-type interference structure.  

We then apply these indicators across a diverse suite of datasets to instantiate the proposed Correlation–Complexity Map and empirically delineate an IQP-compatible regime (\text{high} $I_\text{QCLI}$, high $I_\text{CCI}$).  
Strikingly, we find that classical turbulence datasets occupy this regime, exhibiting both relatively strong beyond-pairwise dependence and correlation-order spectra that deviate markedly from classical binomial randomness.  

To bridge the gap emphasized in prior discussions of practical generative quantum utility \citep{huang2025generative}, we further use a float-to-bitstring representation that converts continuous turbulence fields into fixed-length binary strings via quantization and integer-to-binary encoding, enabling IQP-based modeling while allowing deterministic decoding back to continuous coordinates.

{To make this framework concrete, we focus on a Turbulence dataset that is \emph{a priori} expected to exhibit strong multiscale, long-range dependencies, and leads to broad applications such as fluid dynamics \citep{khojasteh2022lagrangian, gourianov2022quantum, pisoni2025compression}.}
We validate the map’s predictive value by training IQP-based generative models on turbulence using the ``train on classical, deploy on quantum'' workflow, and we introduce a novel latent-parameter adaptation mechanism that reuses a compact IQP circuit across a temporal sequence of turbulence snapshots by learning a low-dimensional trajectory in parameter space and interpolating to generate unseen time steps. In comparative evaluations, the IQP approach demonstrates strong sample and parameter efficiency, e.g., learning from 11 turbulence snapshots with a small latent {set of IQP parameters}, while remaining competitive with (and in certain metrics matching or surpassing) substantially more data-hungry classical baselines. Importantly, these results are positioned as a step toward utility rather than a definitive end-to-end quantum advantage claim: \emph{identifying real datasets whose structure aligns with the inductive bias of IQP circuits is a practical prerequisite for meaningful generative quantum advantage when scaled quantum sampling becomes available.}

Our main contributions are: (i) two complementary, computable indicators ($I_\text{QCLI}$ and $I_\text{CCI}$) for characterizing dataset ``IQP-compatibility'' and ``beyond-pairwise classical dependence'', (ii) a correlation–complexity map that guides the search for IQP-compatible real-world data, and (iii) an empirical case study through a novel latent adaptation method on the turbulence dataset, demonstrating that high-($I_\text{QCLI}$, $I_\text{CCI}$) datasets can exhibit behavior consistent with inductive-bias matching for IQP-type generators. An overview of this work can be found in Fig.~\ref{fig:scheme}.

\begin{figure*}[ht]
\centering
\includegraphics[width=0.8\linewidth]{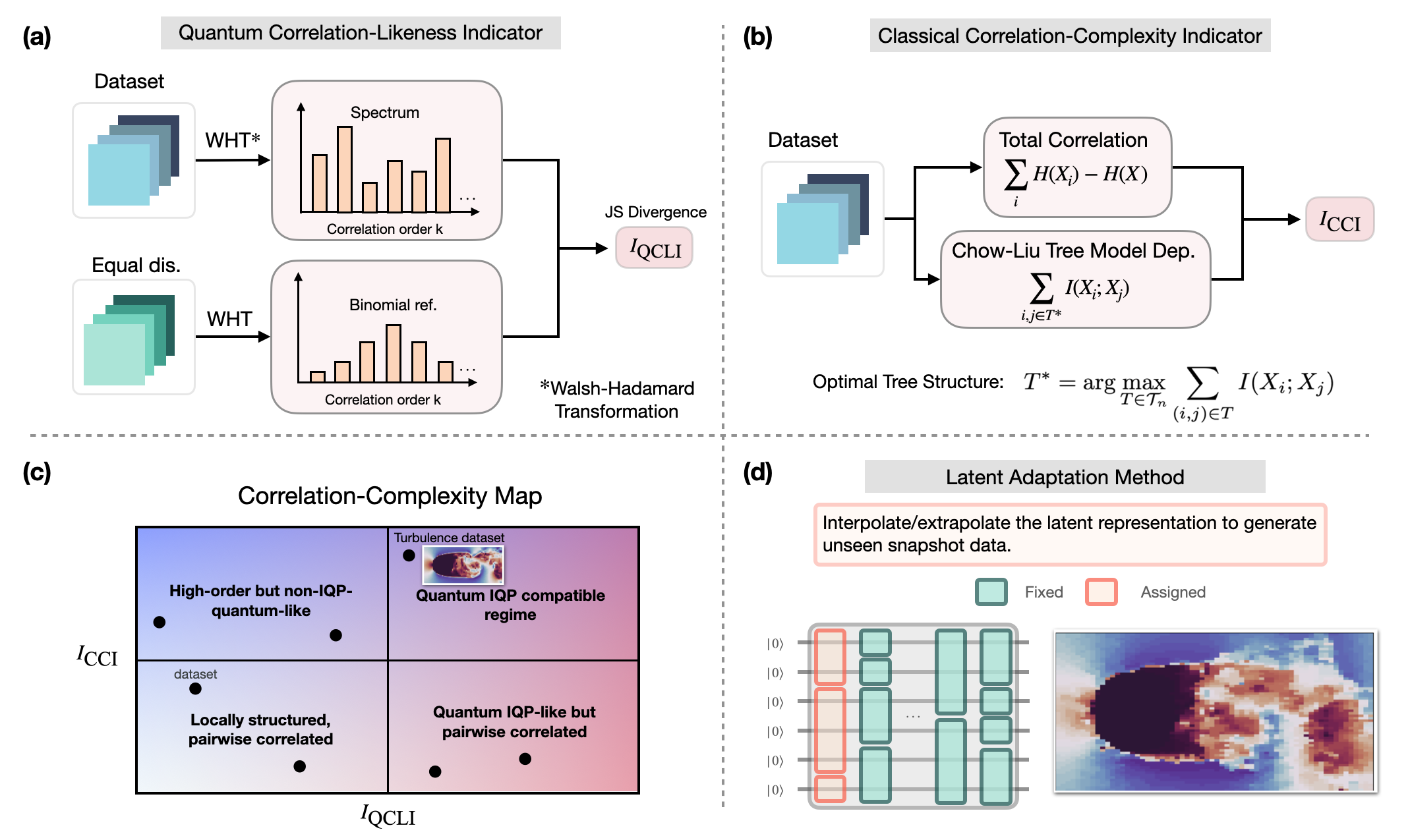}
\caption{Overview of the proposed correlation indicators and the resulting correlation--complexity map.
(a) \textbf{Quantum Correlation--Likeness Indicator (QCLI).} Given a dataset of $n$-bit samples, we compute its correlation-order spectrum by aggregating Walsh--Hadamard (Fourier--Walsh) power over subset sizes $k=|s|$. QCLI, $I_{\text{QCLI}}$, is defined as the Jensen--Shannon divergence between this empirical order spectrum and the binomial reference spectrum induced by the uniform (classically random) distribution.
(b) \textbf{Classical Correlation--Complexity Indicator (CCI).} For the same dataset, we compute the total correlation (multi-information) and the dependence captured by the optimal Chow--Liu tree, obtained by maximizing $\sum_{(i,j)\in T} I(X_i;X_j)$ over spanning trees $T$. CCI, $I_{\text{CCI}}$, is the fraction of total correlation not explained by the optimal tree-structured model.
(c) \textbf{Correlation--Complexity Map.} Datasets are shown in the $(I_{\text{QCLI}}, I_{\text{CCI}})$ plane, separating regimes that are locally structured and classically expressible (low $I_{\text{CCI}}$), quantum-like but classically easy, high-order but non-IQP-aligned, and the high-$I_{\text{QCLI}}$/high-$I_{\text{CCI}}$ regime that is most compatible with IQP-type generative inductive biases. The turbulence dataset is highlighted as residing in this IQP-compatible quadrant, motivating the IQP-based generative modeling experiments. (d) \textbf{Latent adaptation for temporal turbulence synthesis.} Starting from a trained IQP generator, we fix a shared core parameter block and assign (adapt) a low-dimensional latent block to represent different time snapshots. Interpolating or extrapolating the learned latent trajectory enables generation of unseen turbulence snapshots while keeping the circuit architecture and core parameters fixed.}
 \label{fig:scheme}
\end{figure*}

\section{Quantifying IQP compatibility and beyond-pairwise dependence}
In quantum computing and QML, practical quantum advantage depends not only on complexity-theoretic evidence that the underlying problem is asymptotically hard for classical algorithms, but also on whether the practically relevant instances have structure that the quantum method can exploit \citep{buhrman2025formal}.
Prior work has proposed a range of measures and proxy quantities for quantum circuits and states, such as circuit depth, gate count, entanglement, or preparation complexity, to estimate how difficult the corresponding circuits or output distributions are to simulate classically and to characterize the resources needed for universal quantum computation. Here our emphasis is different: we seek data-derived quantities, computable directly from empirical datasets, that assess compatibility with generative circuit families such as IQP.

\paragraph{Indicators derived from classical simulation cost.}
One established way to assess the practical difficulty of a quantum circuit family is to estimate the classical resources needed to simulate its output distribution. In tensor-network methods, for example, this cost depends on the contraction structure of the circuit, such as induced treewidth or contraction order, and therefore yields instance-dependent proxies for classical simulatability. Recent work on random circuit sampling in arbitrary geometries develops this viewpoint explicitly and introduces an \emph{asymptotic complexity density}, a quantity linked to tensor-network contraction cost that helps identify regimes in which random circuits become difficult to simulate classically \citep{decross2025computational}.

These simulation-cost indicators are informative, but they are inherently circuit-centric: they require an explicit description of the quantum circuit and assess hardness through the cost of classically simulating that circuit or its output distribution. Our question is different. Given only samples from a real-world dataset, can we determine whether the dataset has structural features that align with the inductive biases of a classically hard quantum generative family? This dataset-level compatibility question is not addressed by simulation-cost indicators.

\paragraph{State- and entanglement-based diagnostics.}
A related line of work studies properties of quantum states and dynamics that are informative about classical simulatability. In particular, limited entanglement growth can make certain quantum evolutions efficiently simulable by classical methods \citep{vidal2003efficient}. At the same time, entanglement alone is not a reliable predictor of classical hardness: for example, Clifford circuits can generate highly entangled states while still admitting efficient classical simulation under the Gottesman–Knill theorem \citep{aaronson2004improved}. This has motivated more refined diagnostics that aim to capture features of quantum states beyond entanglement alone, especially those associated with non-stabilizer structure.

\paragraph{From circuit-centric diagnostics to dataset-level compatibility indicators.}
Despite this progress, many application settings begin not with a quantum circuit or a prepared quantum state, but with an empirical dataset of classical origin, such as sensor measurements, scientific data, or images. In such settings, circuit- and state-based diagnostics are not directly applicable, because no quantum description is given \emph{a priori}. What is needed instead are indicators that can be computed directly from samples and that answer two distinct questions: first, whether the dataset exhibits structure that is naturally expressible by a target hard-to-sample-from quantum generative family, such as the parity-interference structure associated with IQP circuits; and second, whether the dataset contains dependence structure that remains nontrivial from a classical modeling perspective, beyond pairwise or tree-structured correlations.

This motivates the two indicators introduced in this work, QCLI and CCI. Both are defined directly at the dataset level, so they can be evaluated on empirical data without first specifying a quantum circuit. At the same time, they remain operationally connected to the inductive bias of IQP-type generative models and to the question of representational alignment between a dataset and a classically hard quantum circuit family.
Our goal in this section is therefore to formalize two complementary dataset-level indicators: one that measures how well a dataset aligns with the structural biases of classically hard quantum generative models, and one that measures the extent to which the dataset exhibits genuinely higher-order dependence beyond pairwise structure.

We develop the first indicator through the lens of IQP circuits \citep{shepherd2009temporally, leontica2024exploring}, a standard family of commuting quantum circuits that has played a central role in quantum advantage proposals and whose output distributions naturally exhibit parity-structured interference patterns. An $n$-qubit IQP circuit can be written as:
\begin{equation}
U
= \exp\!\Bigl(i\!\!\sum_{s\subseteq [n]} \theta_s X_s\Bigr)
= H^{\otimes n}
\Bigl(
    \exp\!\bigl(i\!\!\sum_{s\subseteq [n]} \theta_s Z_s\bigr)
\Bigr)
H^{\otimes n},
\label{eq:iqp_unitary}
\end{equation}
where $X_s := \prod_{j\in s} X_j$ and $Z_s := \prod_{j\in s} Z_j$
denote multi-body Pauli strings acting on the subset $s\subseteq[n]$.
Because all $X_s$ commute, this circuit family is simultaneously diagonalizable in the $X$-basis.

Each subset $s$ corresponds to a specific combination of qubit indices, 
for example $s=\{1\}$ gives a single-qubit rotation $X_1$, 
$s=\{2,4\}$ gives a two-qubit interaction $X_2X_4$, 
and $s=\{1,3,5\}$ represents a three-qubit correlation term $X_1X_3X_5$. 
Hence, the sum over all $s\subseteq[n]$ in Eq.~(\ref{eq:iqp_unitary}) 
covers all possible $k$-body interactions ($1\le k\le n$), 
weighted by their corresponding parameters $\theta_s$. 
In practical implementations, however, one typically realizes only a restricted subset of terms, most commonly one- and two-body interactions, due to limited hardware connectivity and resource constraints; implementing all $2^n$ parity terms would require $2^n$ independent parameters $\{\theta_s\}$, which is prohibitively large for learning.
This flexibility means that both sparse and fully connected versions of the IQP model  preserve the same commutative structure and are still considered valid IQP circuits.

\subsection{Measurement distribution and Walsh--Fourier structure}
The measurement outcome in the computational ($Z$) basis (detailed in Appendix~\ref{sec:ddt_1}) is
\begin{eqnarray}
p(z)
&=& \bigl|\langle z|U|0\rangle^{\otimes n}\bigr|^2 \nonumber \\
&=&
\left|
    \frac{1}{2^n}
    \sum_{x\in\{0,1\}^n}
        (-1)^{z\cdot x}
        e^{i\sum_s \theta_s \chi_s(x)}
\right|^2,
\label{eq:iqp_walsh}
\end{eqnarray}
where $\chi_s(x)=(-1)^{\oplus_{j\in s}x_j}$ is the parity (Walsh) character.
Hence the IQP output probabilities
\begin{equation}
p(z)=\bigl|2^{-n}\!\sum_x (-1)^{z\cdot x} e^{i\phi(x)}\bigr|^2,
\quad
\phi(x)=\sum_s \theta_s\chi_s(x),
\end{equation}
are precisely the \emph{squared modulus of the Walsh--Hadamard transform (WHT)} of the phase function $f(x)=e^{i\phi(x)}$ \citep{park2025complexity, bremner2011classical}.
This reveals that the quantum interference pattern of an IQP circuit manifests as the distribution of power among the Fourier--Walsh coefficients of $f$.

\subsection{Correlation order decomposition and high-order dependence}
\label{sec:codhoc}
We consider a dataset of $M$ binary samples
$\{x^{(i)}\}_{i=1}^{M}$ where
$x^{(i)} \in \{0,1\}^n$.
The empirical probability mass function (pmf) is
\begin{equation}
 p(x) = \frac{1}{M}\sum_{i=1}^{M}
\mathbf{1}\!\left\{x^{(i)}=x\right\}.   
\end{equation}
To examine whether an IQP circuit can serve as a suitable inductive bias for generating such a dataset, we apply the WHT to $p(x)$ and square the result, obtaining its Walsh–Fourier spectrum indexed by subsets $s \subseteq [n]$. Let $P(s)$ denote the normalized squared Fourier power associated with subset $s$:
\begin{eqnarray}
\label{eq:ps}
P(s)
&=&
\Bigl|2^{-n}\!\sum_x p(x)(-1)^{\oplus_{i\in s}x_i}\Bigr|^2
\in[0,1]. 
\end{eqnarray}
We then define the \emph{order–$k$ power fraction}, obtained by grouping the spectral power according to its correlation order $|s| = k$:

\begin{equation}
\label{eq:mk}
m_k
=
\frac{
\sum_{|s|=k} P(s)
}{
\sum_{s\subseteq[n]} P(s)
},
\qquad
k=0,1,\ldots,n,
\end{equation}
so that $\sum_{k=0}^n m_k = 1$.
The vector $m = (m_0,\dots,m_n)$ summarizes how correlation power
is distributed across interaction orders. For fully independent and identically distributed fair bitstrings, the spectral power concentrates near the middle orders and follows a binomial reference distribution:
$b_k = \frac{\binom{n}{k}}{2^n}$.
This serves as a classical randomness baseline, which will later be used to construct the correlation indicators described in the next section. Details of the empirical estimation procedure used in our implementation are provided in Appendix~\ref{app:empirical-qcli}.

\subsection{Quantum Correlation--Likeness Indicator (QCLI)}
We measure the discrepancy between the empirical order distribution $m$
and the classical binomial reference $b$
using the Jensen--Shannon divergence \citep{menendez1997jensen}:
\begin{eqnarray}
&&D_\mathrm{JS}(m\Vert b)
=\tfrac12 D_{\mathrm{KL}}\!\big(m\Vert M\big) 
+\tfrac12 D_{\mathrm{KL}}\!\big(b\Vert M\big) \nonumber\\
&&:= I_{\text{QCLI}}(p)
\quad,0 \le I_{\text{QCLI}}(p) \le 1,
\label{eq:qcli}
\end{eqnarray}
where $D_{\mathrm{KL}}$ is the Kullback--Leibler divergence \citep{kullback1951kullback}, and $M=\tfrac12(m+b)$.
Here we explicitly write $I_{\text{QCLI}}(p)$ to emphasize that the indicator
quantifies a property of the underlying distribution $p$.
For notational simplicity, we will suppress the dependence on $p$ and write
$I_{\text{QCLI}}$ in the following.
$I_{\text{QCLI}}$ evaluates how far the \emph{shape} of correlation 
order spectrum deviates from i.i.d.\ classical randomness.
Large values indicate structured interference patterns that are 
not compatible with binomial noise. This quantity has several good properties to be used as an indicator:
\begin{enumerate}
    \item \textbf{Correct null behavior.} If the data have no structured parity content beyond i.i.d. noise, then $m \approx b$ in expectation and $I_{\text{QCLI}} \rightarrow 0$.
    \item \textbf{Permutation invariance.} Orders ignore which bits participate; only the order distribution matters. This matches IQP’s combinatorial ``order'' view.
    \item \textbf{Bounded scale.} $D_\mathrm{JS}(m\Vert b)$ is bounded in $[0,1]$, while $D_\mathrm{KL}(m\Vert b) = \sum_i m_i\log\frac{m_i}{b_i}
     \in [0, n]$ since for the binomial baseline $b_k=\tfrac{\binom{n}{k}}{2^n}$,
    the smallest probability $b_{\min}=2^{-n}$ implies that the
    maximum $D_\mathrm{KL}$ grows linearly with system size $n$.
    \item \textbf{Scale-free across $n$.} Associate with the bounded scale, because both $m$ and $b$ live on $\{0,\dots,n\}$, the statistic compares shapes, not absolute mass at any specific order. This lets us compare datasets of different bit-lengths.
\end{enumerate}
It can be shown that, to second order, 
$D_\mathrm{JS}(m\Vert b)$ behaves as a weighted $\ell_2$ distance 
around the binomial baseline (Appendix~\ref{app:JS_distance}): 
\begin{equation}
D_{\mathrm{JS}}(m\Vert b) \;\approx\; \frac{1}{8\ln 2}\sum_{k=0}^n 
\frac{(m_k - b_k)^2}{b_k}.
\end{equation}
Consequently, by defining the total variation distance 
$\mathrm{TV}(m,b)=\tfrac12\sum_k |m_k-b_k|$, 
and applying the Cauchy--Schwarz inequality,
\begin{align}
\sum_k |m_k - b_k|
&= \sum_k |\delta_k|
\;\le\;
\Bigl(\sum_k \frac{\delta_k^2}{b_k}\Bigr)^{1/2}
\Bigl(\sum_k b_k\Bigr)^{1/2}
\nonumber\\
&= \Bigl(\sum_k \frac{\delta_k^2}{b_k}\Bigr)^{1/2},
\end{align}
where we have used $\sum_k b_k = 1$. 
It follows that
\begin{align}
\mathrm{TV}(m,b)
&\le \tfrac12\sqrt{\sum_k \frac{\delta_k^2}{b_k}}
\;\approx\;
\sqrt{(2\ln 2) D_{\mathrm{JS}}(m\Vert b)}.
\end{align}
We obtain the total variation bound of our correlation indicator $I_{\text{QCLI}}$:
\begin{eqnarray}
\mathrm{TV}(m,b) \lessapprox O(\sqrt{I_{\text{QCLI}}}).
\end{eqnarray}
This relation provides an interpretable link between the spectral and probabilistic distances: the square root dependence implies that datasets with moderately $I_{\text{QCLI}}$ already deviate significantly from binomial randomness in total variation. In other words, $I_{\text{QCLI}}$ quantifies how much of the empirical correlation structure cannot be explained by low-order i.i.d.\ noise, serving as a ``quantum-likeness'' measure of the data’s parity spectrum. 
{We use ``likeness'' operationally: $I_{\text{QCLI}}$ measures deviation from a binomial Walsh-order null in a way motivated by IQP parity interference. It should not be interpreted as a certificate that a distribution is quantum-generated.}
High values of $I_{\text{QCLI}}$ correspond to spectra exhibiting strong constructive or destructive interference patterns that are naturally compatible with IQP-type output structures, while small values indicate that the dataset remains close to the classical combinatorial baseline. Because of its bounded and scale-free nature, $I_{\text{QCLI}}$ thus provides a universal, dimension-agnostic indicator of how well a dataset aligns with the inductive bias of IQP circuits. 

A useful way to interpret $I_{\text{QCLI}}$ is to view $\{m_k\}_{k=0}^n$ as a coarse ``frequency-domain'' fingerprint of the dataset, where $m_k$ reports how much correlation power resides at interaction/correlation order $k$. In Fig.~\ref{fig:dwave_qcli_mk}, we evaluate the $I_{\text{QCLI}}$ of the D-Wave bitstring dataset \citep{scriva2023accelerating} generated from a quantum annealer with 100 qubits collected at anneal times $\tau\in\{1,10,100\}\,\mu\text{s}$. {Since these samples are produced by a physical quantum annealer, they provide a hardware-derived reference for $I_{\text{QCLI}}$, We expect $m_k$ to increase with longer annealing times because the samples concentrate more on lower-energy states (thus more structured), offering a concrete way to examine how the indicator responds to changes in the effective correlation/interference structure of the output distribution.}  For an i.i.d.-like distribution ($\tau=1\,\mu\text{s}$), this fingerprint is smooth and well-approximated by the binomial reference $\{b_k\}$ (as we have mentioned above), so $|m_k-b_k|$ is small and diffuse across $k$, yielding a low $I_{\text{QCLI}}$. In contrast, when the data exhibits coherent high-order structure, the deviations from $b_k$ concentrate into distinct peaks and troughs at specific correlation orders (here, $\tau=100\,\mu\text{s}$), producing a larger $I_{\text{QCLI}}$. This ``spiky'' and alternating pattern is precisely what one expects from IQP-type generative mechanisms: IQP output probabilities arise from interference over many parity components, leading to order-selective enhancement (constructive interference) and suppression (destructive interference). Consequently, high $I_{\text{QCLI}}$ corresponds to spectra with pronounced interference-like structure, and therefore potentially indicates stronger inductive-bias compatibility with IQP-style output distributions, as we will show in the following Sec.~\ref{sec:qcli-mmd}. 

\begin{figure}[ht]
\centering
\includegraphics[width=\linewidth]{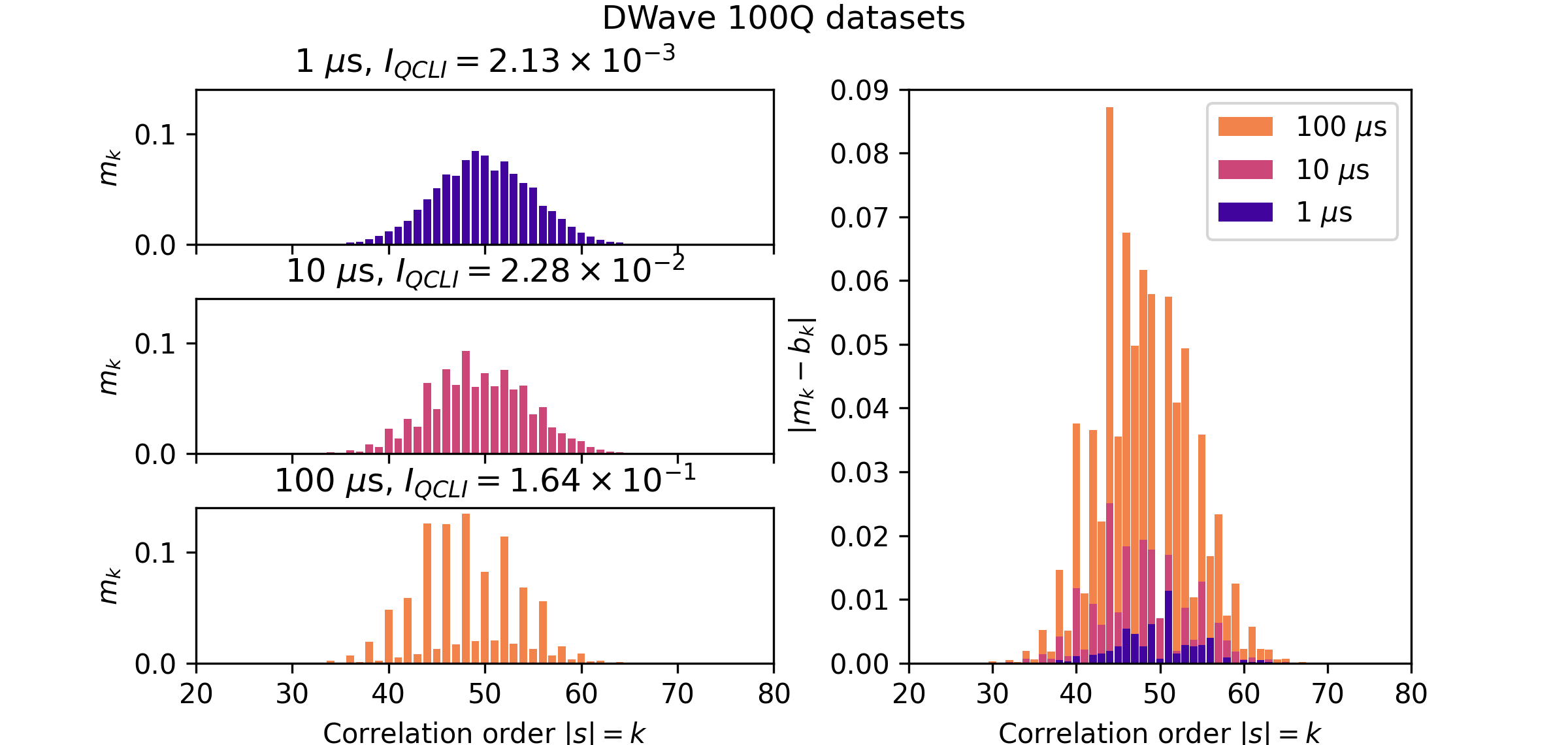}
\caption{
Correlation-order spectra and $I_{\text{QCLI}}$ on D-Wave 100Q datasets \citep{scriva2023accelerating}.
Left: the empirical correlation-order spectrum $m_k$ (spectral mass aggregated over Walsh–Hadamard parities of order $|s|=k$) for three D-Wave 100Q datasets collected at anneal times $\tau\in\{1,10,100\}\,\mu\text{s}$, with the corresponding $I_{\text{QCLI}}$ shown in the titles. Right: the absolute deviation $|m_k-b_k|$ from the i.i.d. ideal binomial baseline spectrum $b_k$. As $\tau$ increases, $I_{\text{QCLI}}$ rises and the spectra become visibly more structured, exhibiting sharper peaks and pronounced cancellations across correlation orders, consistent with stronger constructive/destructive interference patterns.
}
 \label{fig:dwave_qcli_mk}
\end{figure}

\subsection{Classical Correlation--Complexity Indicator (CCI)}
While $I_{\text{QCLI}}$ captures how the spectral structure of a dataset
deviates from classical binomial randomness in the \emph{Fourier domain},
we also wish to quantify how much of its dependence structure is irreducible to pairwise models in the \emph{probability domain}.
For this purpose we define the \emph{Classical Correlation--Complexity Indicator}:
\begin{eqnarray}
&&I_{\text{CCI}}
= 1 - \frac{I_{\mathrm{TC}}^{\text{tree}}}{I_{\mathrm{TC}}}, \nonumber\\
&&I_{\mathrm{TC}} = \sum_{i=1}^n \mathcal{H}(X_i) - \mathcal{H}(X_1, X_2, \cdots ,X_n),
\label{eq:icci_def}
\end{eqnarray}
where $\mathcal{H}(X)=-\sum_{x\in X}p(x)\log p(x)$ is the entropy and $I_{\mathrm{TC}}$ is the total correlation (also called multivariate mutual
information) measuring the overall statistical dependence among all $n$ variables, originating from the classical information-theoretic notion of \emph{multi-information} introduced by \citet{watanabe1960information} and further developed in \citet{schneidman2003network} to study collective correlations in binary neural populations.
And $I_{\mathrm{TC}}^{\text{tree}}$ is the portion of this dependency that can be
captured by a maximum-likelihood Chow--Liu tree model~\citep{chow1968approximating}:
\begin{equation}
I_{\mathrm{TC}}^{\text{tree}}
= \sum_{(i,j)\in T^\ast} I(X_i;X_j),
\end{equation}
with $T^\ast$ being the optimal dependency tree
constructed from pairwise mutual information $I(X_i;X_j)$.

\paragraph{Computation via maximum-likelihood Chow--Liu tree.}
Given a dataset $\{x^{(m)}\}_{m=1}^M$ of $n$ binary variables
$X=(X_1,\dots,X_n)$, we estimate pairwise marginals
$p_{ij}(x_i,x_j)$ and univariate marginals $p_i(x_i)$.
The pairwise mutual information matrix is then
\begin{equation}
I(X_i;X_j)
= \sum_{x_i,x_j} p_{ij}(x_i,x_j)
    \log \frac{p_{ij}(x_i,x_j)}{p_i(x_i)p_j(x_j)}.
\end{equation}
The Chow--Liu algorithm constructs the \emph{maximum-weight spanning tree}
over $n$ nodes, where each edge $(i,j)$ is assigned weight
$w_{ij} = I(X_i;X_j)$.
The optimal tree structure $T^\ast$ maximizes the total pairwise information:
\begin{equation}
T^\ast = \arg\max_{T\in \mathcal{T}_n} 
\sum_{(i,j)\in T} I(X_i;X_j),
\end{equation}
where $\mathcal{T}_n$ denotes all spanning trees over $n$ variables.
Once the optimal tree $T^\ast$ is found,
the best tree-structured approximation of the joint distribution is
\begin{equation}
p_T(x)
= \prod_{i=1}^{n} p_i(x_i)
  \prod_{(i,j)\in T^\ast}
  \frac{p_{ij}(x_i,x_j)}{p_i(x_i)p_j(x_j)}.
\end{equation}
The corresponding total correlation captured by the tree model is
\begin{equation}
I_{\mathrm{TC}}^{\text{tree}}
= \sum_{(i,j)\in T^\ast} I(X_i;X_j),
\end{equation}
and the residual dependence beyond pairwise structure follows as
\begin{equation}
I_{\text{CCI}}
= 1 - \frac{I_{\mathrm{TC}}^{\text{tree}}}{I_{\mathrm{TC}}},
\qquad
I_{\mathrm{TC}} = \sum_i \mathcal{H}(X_i) - \mathcal{H}(X).
\end{equation}

\paragraph{Information--theoretic interpretation.}
Chow and Liu~\citep{chow1968approximating} proved that
the tree distribution $p_T(x)$ obtained above
maximizes the likelihood of the data among all possible tree-structured
models, or equivalently minimizes the Kullback--Leibler divergence:
\begin{equation}
T^\ast
= \arg\min_{T\in\mathcal{T}_n}
D_{\mathrm{KL}}\!\big(p(x)\,\Vert\,p_T(x)\big).
\end{equation}
Moreover, the KL divergence between the true joint and its
tree approximation decomposes as
\begin{equation}
D_{\mathrm{KL}}\!\big(p(x)\Vert p_T(x)\big)
= I_{\mathrm{TC}} - I_{\mathrm{TC}}^{\text{tree}}.
\end{equation}
Hence $I_{\mathrm{TC}}^{\text{tree}}$ represents the portion of
total correlation retained by the optimal pairwise model,
and the residual $I_{\mathrm{TC}} - I_{\mathrm{TC}}^{\text{tree}}$
quantifies the irreducible multi-body dependency lost by tree approximation.

$I_{\text{CCI}}$ therefore measures the
\emph{fraction of total correlation that is irreducible to pairwise structure}.
By construction ($D_{\mathrm{KL}}\!\big(p(x)\Vert p_T(x)\big) \ge 0$ and hence $I_{\mathrm{TC}} \ge I_{\mathrm{TC}}^{\text{tree}} \quad \forall I_{\mathrm{TC}}^{\text{tree}}$),
\[
0 \le I_{\text{CCI}} \le 1,
\]
with small values indicating data that are well explained by local pairwise
dependencies, and large values indicating complex, nonlocal interactions.
When combined with the spectral indicator $I_{\text{QCLI}}$,
the pair $(I_{\text{QCLI}}, I_{\text{CCI}})$
forms a two-dimensional \emph{quantum--classical relevance map}:
datasets with high scores on both axes are those that simultaneously
deviate from classical combinatorial randomness and defy pairwise
graphical modeling, a regime in which generative quantum models such as
IQP circuits may exhibit genuine advantage in parameter and data efficiency.

Having defined $I_{\text{QCLI}}$ and $I_{\text{CCI}}$ as complementary axes of the Correlation–Complexity Map, we now explain why $I_{\text{QCLI}}$ is not merely descriptive, but could be predictive for IQP generative modeling. Specifically, we connect $I_{\text{QCLI}}$ to an irreducible approximation floor for truncated IQP families under a Maximum Mean Discrepancy (MMD) objective \citep{gretton2012kernel}, and then validate this support-mismatch mechanism empirically. We further discuss the structural implications of high-$I_\text{QCLI}$ regimes and their relationship to beyond-pairwise dependence captured by $I_{\text{CCI}}$.

\subsection{$I_\text{QCLI}$ as a Conditional Proxy for Architecture--Data Support Alignment}
\label{sec:qcli-mmd}

To formalize the notion of \emph{architecture--data support alignment} underlying $I_\text{QCLI}$, we consider a restricted IQP family whose expressive degrees of freedom are controlled by an interaction-order parameter $d$. Intuitively, smaller $d$ limits which parity components the model can directly parameterize. As a result, if the target distribution places substantial Walsh--Fourier mass outside the architecture-accessible support, this mass creates a representational mismatch that cannot be removed by optimization alone.

\begin{definition}[Order-$d$ IQP circuit]
An IQP circuit is called an \emph{order-$d$} IQP circuit if its commuting Hamiltonian contains only Pauli-$Z$ strings of Hamming weight at most $d$. Equivalently, it has the form
\[
U
=
H^{\otimes n}
\exp\!\Bigl(
i\!\!\sum_{\substack{s\subseteq[n]\\ |s|\le d}}
\theta_s Z_s
\Bigr)
H^{\otimes n},
\]
where each term $Z_s := \prod_{j\in s} Z_j$ acts on at most $d$ qubits. This restriction controls which parity terms appear in the phase polynomial and therefore which Walsh--Fourier components the model can directly tune within this family.
\end{definition}

The complete proof of the following proposition is provided in Appendix~\ref{app:proof-of-qcli-mmd}.

\begin{proposition}[Walsh-support mismatch induces an MMD floor]
\label{prop:walsh-support-mmd-floor}
Let $p$ be a binary distribution over $\{0,1\}^n$, and let $\mathcal{F}_{\mathrm{IQP}}$ be a restricted IQP architecture with tunable Walsh--Fourier support $S_A\subseteq 2^{[n]}$. Let $\mathcal{L}_{\mathrm{MMD}}(p,q)$ denote the MMD loss between $p$ and $q$ under a positive kernel. For example, for a Gaussian kernel $k(b,b')$ on bitstrings,
\begin{eqnarray}
\mathcal{L}_{\mathrm{MMD}}(p,q)
&=&
\mathbb{E}_{b,b' \sim p}\!\big[k(b,b')\big]
-2\,\mathbb{E}_{b\sim p,\; b'\sim q}\!\big[k(b,b')\big] \nonumber \\
&&+
\mathbb{E}_{b,b'\sim q}\!\big[k(b,b')\big].
\label{eq:mmd-raw}
\end{eqnarray}
Let \[ P(S) := \bigl| 2^{-n}\!\sum_x p(x)\,\chi_S(x) \bigr|^2, \quad \chi_S(x)=(-1)^{\oplus_{j\in S} x_j}, \] denote the Walsh--Fourier power of $p$ at frequency $S$. Then there exists a constant $C_A>0$, depending only on the kernel normalization and the chosen IQP architecture, such that
\begin{equation}
\label{eq:support-mmd-floor}
\min_{q\in\mathcal{F}_{\mathrm{IQP}}}
\mathcal{L}_{\mathrm{MMD}}(p,q)
\;\ge\;
C_A
\sum_{S\notin S_A} P(S).
\end{equation}

Consequently, any Walsh--Fourier power of the target distribution outside the architecture-accessible support contributes an irreducible term to the MMD objective.
\end{proposition}

\begin{assumption}[QCLI--support alignment]
\label{assump:qcli-support}
Let $p$ be as above, with empirical correlation--order histogram $m=(m_0,\dots,m_n)$ and binomial baseline $b=(b_0,\dots,b_n)$, 
with corresponding $I_{\mathrm{QCLI}}(p) = D_{\mathrm{JS}}(m\Vert b)$ defined in Eq.~(\ref{eq:qcli}).
Let $S_A\subsetneq 2^{[n]}$ denote the Walsh--Fourier frequencies tunable by the chosen restricted IQP architecture. We assume that there exists a constant $\kappa>0$, depending only on the architecture and kernel representation, such that
\begin{equation}
\label{eq:assumption-qcli-support}
\sum_{S\notin S_A} P(S)
\;\ge\;
\kappa\,\bigl(1-I_{\mathrm{QCLI}}(p)\bigr).
\end{equation}
Equivalently, within the dataset and architecture regimes considered here, the deviation of the correlation-order histogram from the binomial baseline is assumed to correlate with how much Walsh--Fourier power lies inside the architecture-accessible support. 
\end{assumption}

Combining Proposition~\ref{prop:walsh-support-mmd-floor} with Assumption~\ref{assump:qcli-support} gives
\begin{equation}
\label{eq:qcli-mmd-lowerbound}
\min_{q\in\mathcal{F}_{\mathrm{IQP}}}
\mathcal{L}_{\mathrm{MMD}}(p,q)
\;\ge\;
C'\,\bigl(1-I_{\text{QCLI}}(p)\bigr),
\end{equation}
where $C'=C_A\kappa>0$ depends only on the kernel representation and the chosen IQP architecture.

\paragraph{Interpretation and scope.}
Proposition~\ref{prop:walsh-support-mmd-floor} formalizes a support-mismatch mechanism: spectral mass outside $S_A$ cannot be matched by the restricted IQP family and hence contributes irreducibly to the MMD objective. The connection to $I_{\mathrm{QCLI}}$ is conditional on Assumption~\ref{assump:qcli-support}, which treats the order-spectrum deviation as a proxy for architecture-accessible Walsh support. Under this alignment assumption, larger $I_{\mathrm{QCLI}}$ indicates less inaccessible Walsh mass and hence a smaller support-induced MMD floor for the chosen architecture. We treat this alignment as an empirical hypothesis and assess it in Sec.~\ref{sec:qcli-mmd-empirical}.

The proposition is intended for the practically relevant truncated regime $S_A\subsetneq 2^{[n]}$, where the architecture cannot tune all parity components. The support set $S_A$ depends on the architecture and on the kernel/Walsh representation used to express the MMD objective. In the full-support case $S_A=2^{[n]}$, the outside mass $\sum_{S\notin S_A}P(S)$ vanishes, so the specific obstruction captured by Eq.~(\ref{eq:support-mmd-floor}) becomes vacuous. This does not imply perfect representability by IQP circuits; it only means that missing Walsh support is no longer the source of this lower bound. Other sources of error, including finite-budget optimization, finite sampling, kernel mismatch, and architecture-intrinsic constraints, may still affect the achieved MMD.

\subsection{Empirically Probing the QCLI–MMD Mechanism}
\label{sec:qcli-mmd-empirical}

We next move from analysis to empirical probing. {While Eq.~(\ref{eq:qcli-mmd-lowerbound}), obtained by combining Proposition~\ref{prop:walsh-support-mmd-floor} with Assumption~\ref{assump:qcli-support}, provides a \emph{conditional lower bound} on the best achievable MMD for a restricted IQP family}, practical training rarely approaches this limit, since optimization is imperfect, finite-sample estimation introduces variance, and quantum (or shot-based) sampling adds additional noise. Consequently, rather than attempting to \emph{tighten} or \emph{achieve} the bound directly, we test the underlying prediction of the support-mismatch mechanism: as $I_{\text{QCLI}}$ of the dataset increases, the \emph{observed} MMD attainable under a fixed learner IQP architecture and training budget should systematically decrease, and the residual error should be consistent with spectral mass lying outside the model's tunable support.

\paragraph{Dataset generation with controlled $I_{\text{QCLI}}$.}
To mimic the practical setting where the target distribution is unknown and may exhibit an unknown (and potentially high-order) parity support, we generate a diverse family of target datasets by sampling $10^4$ bitstrings from $n=16$-qubit IQP circuits with gate locality up to $4$ and varying circuit size across seven gate counts,
\[
G_n^{\text{data}} \in \{140,\,280,\,420,\,560,\,700,\,840,\,1050\}.
\]
To obtain targets with a wide range of spectral structures, we optimize the gate parameters of these \emph{data-generating} circuits using Simultaneous Perturbation Stochastic Approximation (SPSA) with the objective of increasing the $I_{\text{QCLI}}$, and then \emph{filter} the resulting circuits according to their measured $I_{\text{QCLI}}$. This procedure yields IQP-generated datasets whose $I_{\text{QCLI}}$ spans from values close to $0$ up to approximately $0.7$, enabling a controlled study of how architecture--data alignment (as quantified by $I_{\text{QCLI}}$) manifests in the achievable MMD under a fixed learner family. Because the generator locality ($\le 4$) is strictly larger than the learner locality used below ($\le 2$), the learner is guaranteed to face a nontrivial and \emph{a prior} unknown support mismatch, analogous to real-data scenarios.

\paragraph{Fixed-architecture IQP learner with support mismatch.}
For the learner, we fix the circuit architecture to be strictly more restricted than the generator by construction: we consider three settings with gate number
\[
G_n \in \{50,\,100,\,150\},
\]
and impose gate locality up to $2$.
This choice ensures that the learner's \emph{tunable parity support} cannot fully cover that of the order-$4$ data-generating circuits, thereby creating a controlled support mismatch. We train the learner using the ``train on classical, deploy on quantum'' scheme proposed in \cite{recio2025train}, while deferring method details to Appendix~ \ref{app:tocdoq-details}. This controlled setup allows us to examine how representational alignment, as quantified by $I_{\text{QCLI}}$, influences the approximation mismatch captured by the MMD loss.

\paragraph{Empirical probe.}
For each target dataset (indexed by its $I_{\text{QCLI}}$ value), we fit the fixed-architecture learner under the same optimization (\cite{recio2025train}) protocol and training budget (Adam with 5000 iterations and learning rate $10^{-4}$) and record the resulting MMD loss. Fig.~\ref{fig:qcli_mmd} summarizes the relationship between $I_{\text{QCLI}}$ and the achieved MMD across all constructed targets for three learner settings, $G_n\in\{50,100,150\}$. Although individual runs exhibit variability (reflecting finite-budget optimization and estimation noise), the \emph{upper envelope} of the observed MMD decreases systematically as $I_{\text{QCLI}}$ increases.

Moreover, increasing the learner gate count $G_n$ shifts the envelopes downward across $I_{\text{QCLI}}$ values. Empirically, both the \emph{lower envelope} and the \emph{upper envelope} decrease as $G_n$ increases. This is consistent with the support-mismatch mechanism: larger $G_n$ yields a more expressive IQP family and, for our architecture class, typically enlarges the effective tunable support $S_A$, thereby reducing the residual spectral mass $\sum_{S\notin S_A}P(S)$ that drives the irreducible mismatch in {Proposition~\ref{prop:walsh-support-mmd-floor}}. In particular, the reduction of the upper envelope in the low-$I_{\text{QCLI}}$ (mismatch-dominated) regime across different $G_n$ suggests that even when targets exhibit weak IQP-aligned spectral structure, a less truncated learner can partially mitigate leakage outside $S_A$ and therefore reduce the worst-case mismatch observed under finite training.

\begin{figure*}[ht]
\centering
\includegraphics[width=0.9\linewidth]{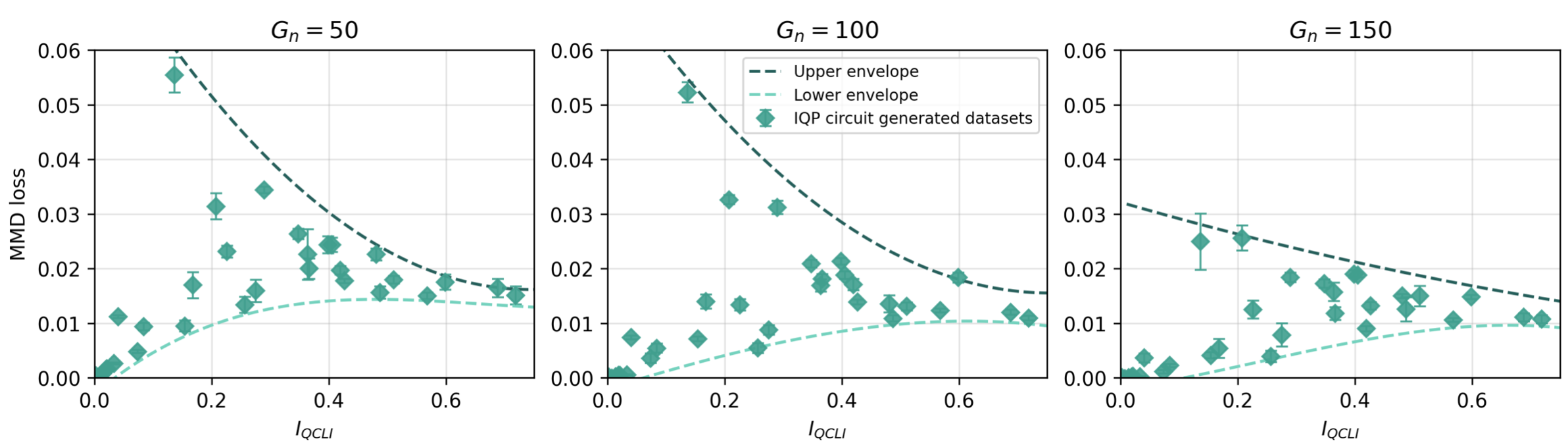}
\caption{Empirical probe of the QCLI–MMD support-mismatch mechanism.
Scatter plots show the MMD loss achieved when fitting a fixed-architecture IQP learner to a collection of IQP-generated datasets spanning a wide range of $I_{\text{QCLI}}$ values, under a fixed training budget. Each panel corresponds to a different learner expressivity setting (gate count $G_n\in\{50,100,150\}$, with restricted locality), thereby enforcing a controlled architecture–data mismatch. Dashed curves indicate the estimated upper and lower envelopes of the observed MMD as a function of $I_{\text{QCLI}}$. Details of the envelope estimation are provided in Appendix~\ref{app:envelopes}.}
 \label{fig:qcli_mmd}
\end{figure*}

The overall behavior is regime-dependent and reflects the interplay between irreducible approximation mismatch and optimization-limited training error. In the very low-$I_{\text{QCLI}}$ regime ($I_\text{QCLI}\lessapprox 0.1$), targets exhibit little apparent parity structure accessible to the learner. Under a fixed training budget, such targets can be fit to small MMD values, leading to a relatively flat lower envelope despite low QCLI. As $I_{\text{QCLI}}$ increases into an intermediate regime, targets develop nontrivial parity structure that remains poorly aligned with the learner’s restricted $S_A$, producing the largest observed mismatches in MMD loss around a value of $I_\text{QCLI} \approx 0.15$. 
At higher $I_{\text{QCLI}}$, two effects become visible. First, the \emph{upper envelope} continues to decrease; once the target’s spectral mass is sufficiently aligned with $S_A$, large support-induced errors become unlikely even under imperfect training. Second, the \emph{lower envelope} no longer rises monotonically. While it may increase initially as the targets become more structured (and hence harder to optimize under a fixed budget), at sufficiently large $I_{\text{QCLI}}$ the lower envelope saturates and can even exhibit a mild downward trend. This is consistent with improved representational alignment reducing the intrinsic approximation floor, so that beyond a certain point the benefit of alignment begins to offset the increased optimization difficulty.

{The joint dependence on $I_{\text{QCLI}}$ and $G_n$ observed in Fig.~\ref{fig:qcli_mmd} is consistent with the support-mismatch mechanism underlying Eq.~(\ref{eq:qcli-mmd-lowerbound})}, low $I_{\text{QCLI}}$ permits substantial spectral leakage outside $S_A$ and thus potentially large mismatch, whereas high $I_{\text{QCLI}}$ constrains this leakage and limits worst-case error; increasing $G_n$ reduces mismatch by reducing truncation of the learner’s effective support. We emphasize that the precise transition points between regimes depend on the data-generation process and training budget, but the qualitative trends are expected to persist across other IQP sources and learners, including practical datasets, as they reflect representational alignment rather than properties of any particular instance.

\subsection{High-$I_{\text{QCLI}}$ IQP Outputs Exhibit Elevated $I_{\text{CCI}}$}
\label{sec:high_qcli_high_cci}

{Proposition~\ref{prop:walsh-support-mmd-floor}, Eq.~(\ref{eq:qcli-mmd-lowerbound})}, and the empirical probe above motivate $I_{\text{QCLI}}$ as a proxy for \emph{architecture--data alignment} within the IQP family; when $I_{\text{QCLI}}$ is larger, a fixed-architecture (truncated) IQP learner is less likely to suffer large support-induced mismatch, and the achievable MMD correspondingly improves. Having established this \emph{approximation} perspective, we now ask a complementary \emph{structural} question: \emph{when an IQP output distribution attains high $I_{\text{QCLI}}$ (i.e., strong IQP-compatible interference structure),
what level of beyond-pairwise dependence
should we expect it to exhibit?} In other words, does ``being more IQP-aligned'' also imply that the distribution becomes \emph{classically more complicated} in the sense of requiring genuinely multivariate dependencies beyond pairwise/tree-structured models?

To address this, we perform controlled experiments in which we directly optimize the parameters of random IQP circuits with the sole objective of maximizing $I_{\text{QCLI}}$. Importantly, no constraint, penalty, or regularization is imposed on $I_{\text{CCI}}$ during optimization; $I_{\text{CCI}}$ is computed only as an auxiliary diagnostic. Across all experiments, we consistently observe an emergent coupling: as $I_{\text{QCLI}}$ increases, the $I_{\text{CCI}}$ of the \emph{same} circuits also increases, despite $I_{\text{CCI}}$ never appearing in the objective. This trend is summarized in Fig.~\ref{fig:high_qcli}.

\begin{figure}[ht]
\centering
\includegraphics[width=\linewidth]{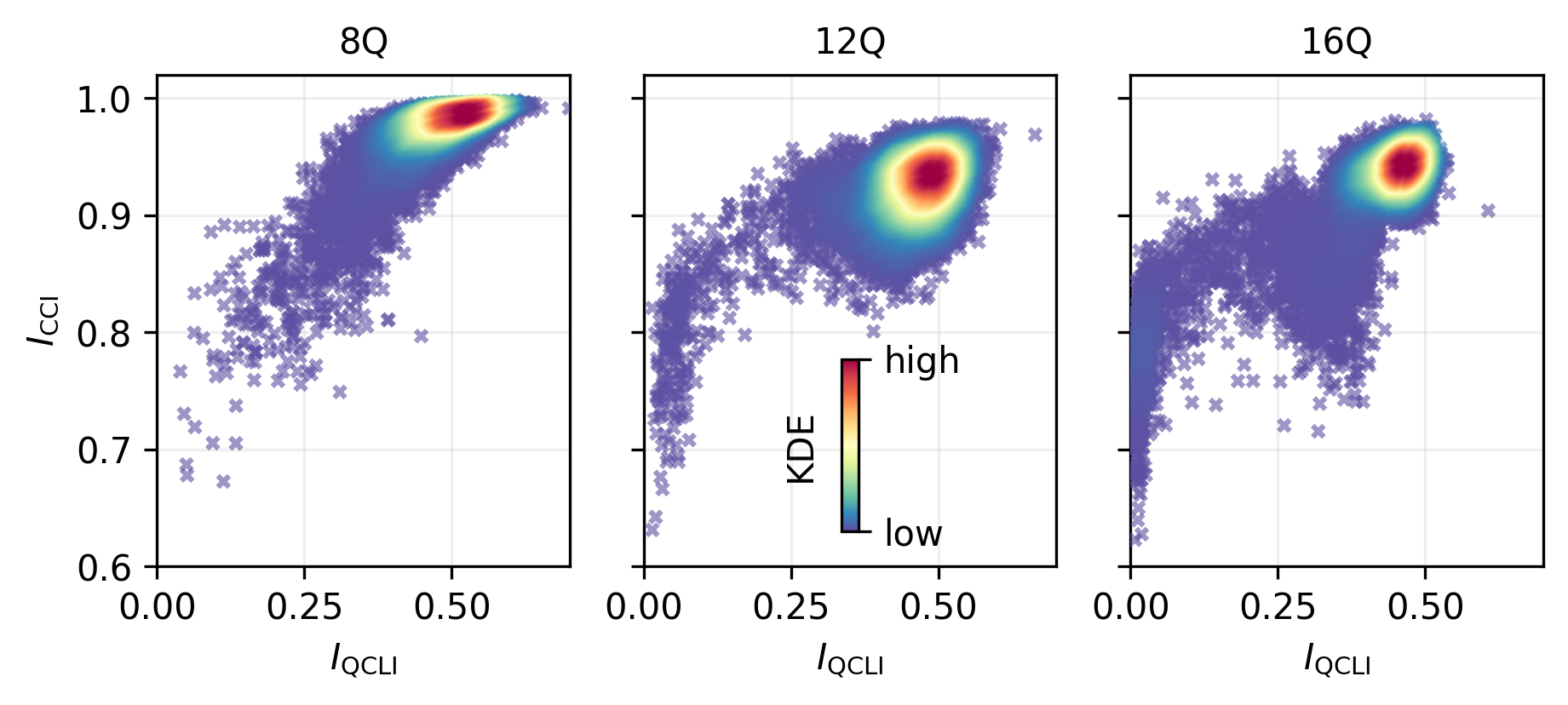}
\caption{$I_{\text{QCLI}}$ versus $I_{\text{CCI}}$ for 2-local IQP circuits with 8Q, 12Q, and 16Q, each using 150 random gates. Circuit parameters are optimized solely to maximize $I_{\text{QCLI}}$, yet higher $I_{\text{QCLI}}$ consistently coincides with elevated $I_{\text{CCI}}$ across all system sizes, indicating an emergent coupling between IQP representational alignment and global multivariate classical correlations.}
 \label{fig:high_qcli}
\end{figure}

Concretely, we perform this study under three system sizes: 8 qubits (8Q), 12 qubits (12Q), and 16 qubits (16Q). For each setting, we consider IQP circuits composed of 150 randomly selected parameterized gates, each gate acting on two qubits (2-local). Starting from random parameter initializations, we apply a SPSA optimization procedure that maximizes $I_{\text{QCLI}}$ alone. To characterize the resulting distribution of dependence measures, we sample 20000 IQP circuits per setting along the optimization process and compute their corresponding $I_{\text{QCLI}}$ and $I_{\text{CCI}}$ values.

Recall that $I_{\text{CCI}}$ quantifies the portion of total dependence that cannot be captured by any optimal tree-structured (pairwise) approximation. Empirically, we find that $I_{\text{CCI}}$ increases systematically as $I_{\text{QCLI}}$ is maximized, even though $I_{\text{CCI}}$ is never included in the objective, indicating that these two notions of structure are not independent within the IQP family.

While $I_{\text{QCLI}}$ governs architecture--data alignment and the resulting irreducible MMD floor under truncated IQP learners, high-$I_{\text{QCLI}}$ IQP outputs also tend to exhibit stronger genuinely multivariate (beyond-pairwise) classical correlations. Taken together, the results suggest a refined structural picture: within IQP-generated distributions, moving toward higher $I_{\text{QCLI}}$ simultaneously increases IQP-compatibility and strengthens dependence beyond what pairwise models can explain.

\subsection{Correlation-Complexity Map for Various Datasets}
\label{sec:Map}

The preceding sections establish two complementary interpretations of our indicators within the IQP lens: $I_{\text{QCLI}}$ acts as a proxy for architecture--data alignment that constrains the irreducible approximation mismatch (as reflected in achievable MMD under fixed IQP learners), and IQP outputs with high $I_{\text{QCLI}}$ tend to exhibit stronger beyond-pairwise dependence, as measured by $I_{\text{CCI}}$. We now move from \emph{mechanism} to \emph{diagnosis}; we compute $(I_{\text{QCLI}}, I_{\text{CCI}})$ for a diverse set of real-world, simulated, and quantum-generated datasets and place them into a common two-dimensional landscape, the \emph{Correlation--Complexity Map} (Fig.~\ref{fig:ccm}). 

We choose datasets according to a simple coverage principle. Specifically, we include (i) standard structured benchmarks (e.g., MNIST and synthetic blobs) that are expected to be largely classically expressible and have appeared in the Ref. \citep{recio2025train}, (ii) quantum-hardware or quantum-sampling datasets (D-Wave annealer samples and Random-Circuit Sampling (RCS)) that provide points for ``quantum-origin'' and interference-like structure, and (iii) dynamical/physical systems with known complicated dependence (Lorenz and turbulence) that are \emph{a priori} plausible candidates for elevated higher-order dependence. In particular, turbulence is widely regarded as exhibiting multiscale structure and long-range statistical dependence. Importantly, the same pipeline can be applied to arbitrary user-provided datasets, and the purpose of the map is to provide a dataset-centric diagnostic rather than a dataset-specific claim.

The $I_{\text{QCLI}}$ axis quantifies the extent of parity-structured, interference-like deviation from the binomial ideal baseline, while the $I_{\text{CCI}}$ axis quantifies the portion of dependence that cannot be explained by the optimal tree-structured (pairwise) model. The map provides an operational guide for identifying dataset regimes that are (or are not) structurally compatible with the inductive biases of IQP-type generators.

To contextualize the map with an IQP reference, we construct an empirical \emph{IQP envelope} by using the data points collected in Sec.~\ref{sec:high_qcli_high_cci} {using the approach provided in Appendix~\ref{app:envelope_CCM}}, that is, a data-driven boundary that approximates the \emph{upper-right reach} of this IQP family in the QCLI--CCI plane under our optimization protocol. We use this envelope as a visual reference for ``IQP-like'' structure regime rather than as a hard feasibility boundary.

{
We emphasize that the map is a decision-support diagnostic rather than a hardness certificate. Its axes are empirical statistics estimated from finite samples, and the shaded IQP region is a protocol-specific reference obtained from one optimized IQP family rather than a hard feasibility boundary. Points in the high-$I_{\mathrm{QCLI}}$/ high-$I_{\mathrm{CCI}}$ region are therefore not guaranteed to exhibit quantum advantage; instead, they identify candidate datasets which align with the inductive bias of the IQP model and which are, thus, worth testing.
}

In this map, datasets clustered near the origin (low $I_{\text{QCLI}}$, low $I_{\text{CCI}}$) exhibit predominantly low-order, classically expressible (by Chow-Liu tree) correlations, representative of structured yet ``classical'' generative tasks such as MNIST or simple geometric blobs. D-Wave quantum annealing samples \citep{scriva2023accelerating} with short anneal times with 100 qubits also fall in this region, reflecting behavior well-explained by local pairwise couplings. Moving upward along the $I_{\text{CCI}}$ axis identifies datasets with increasingly non-local, many-body dependence, such as Lorenz attractor \citep{tucker1999lorenz}, whose structure cannot be faithfully captured by tree-based models. Meanwhile, movement along the $I_{\text{QCLI}}$ axis highlights datasets whose Fourier–Walsh signatures contain quantum-like interference structure distinct from classical randomness, including RCS data from Googles experiments \citep{arute2019quantum}.

The (relatively) upper-right quadrant, high $I_{\text{QCLI}}$ and high $I_{\text{CCI}}$, marks the IQP-compatible regime, where datasets simultaneously exhibit strong parity-structured interference and high-order non-pairwise correlation dependence. The turbulence datasets (D1, D2) \citep{khojasteh2022lagrangian}, where D1 and D2 referring to different simulation domains, occupy this region, suggesting that capacity-limited IQP generators may exhibit improved representational alignment and sample/parameter efficiency relative to classical baselines in this domain, which we investigate in the following sections. This map therefore serves not only as an empirical diagnostic, but also as a principled guide for identifying problem classes where IQP-based quantum generative models are naturally well-aligned.

\paragraph{On the $I_{\text{QCLI}}$ magnitudes.}
We emphasize that the absolute $I_{\text{QCLI}}$ values in Fig.~\ref{fig:ccm} should be interpreted primarily \emph{relatively} within this dataset collection and with respect to the IQP reference envelope. Although $I_{\text{QCLI}}$ is bounded by $1$, the largest $I_{\text{QCLI}}$ values observed among the real and benchmark datasets plotted here are below $\approx 0.2$, substantially smaller than the values achievable by synthetic IQP-generated targets in our controlled studies (which can reach $\sim 0.7$). Importantly, ``moderate'' $I_{\text{QCLI}}$ values on this scale can already correspond to visibly structured spectra, for example, the D-Wave 100Q dataset at $100\,\mu$s attains $I_{\text{QCLI}}\approx 0.16$ and exhibits pronounced order-selective deviations from the binomial baseline in Fig.~\ref{fig:dwave_qcli_mk}.
This gap between typical real datasets and synthetic IQP extremes highlights a practical challenge: strongly IQP-aligned, interference-dominated structure may be rare in off-the-shelf datasets, and searching for IQP-compatible real-world distributions is therefore nontrivial. The purpose of the Correlation--Complexity Map is precisely to make this search operational by providing a calibrated, model-relevant coordinate system, so that progress can be tracked in terms of movement toward the IQP-compatible regime, and candidate domains can be prioritized even when absolute $I_{\text{QCLI}}$ values remain far from the theoretical maximum.

\begin{figure*}[ht]
\centering
\includegraphics[width=0.7\linewidth]{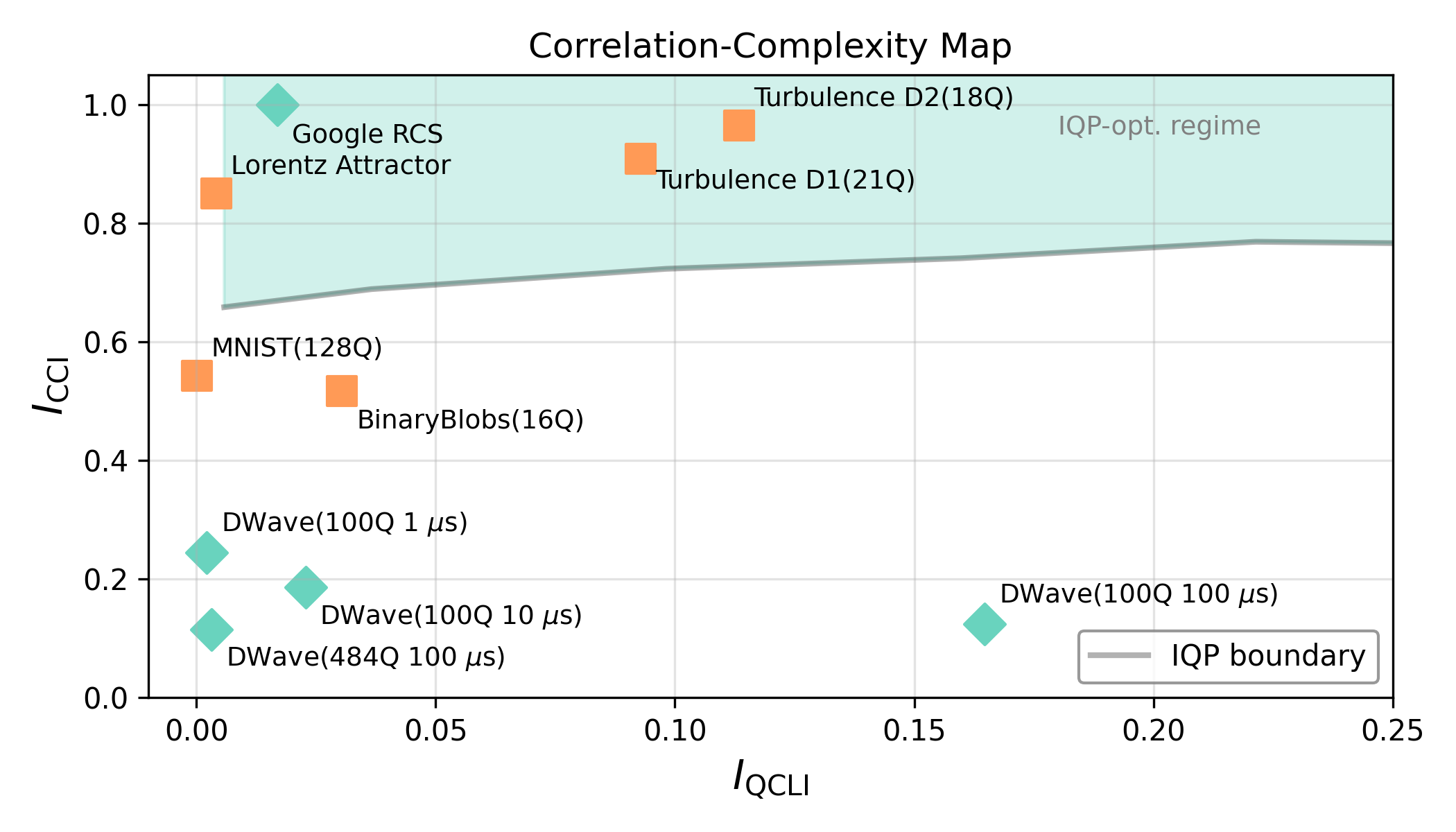}
\caption{Correlation–Complexity Map.
Each dataset is positioned by $(I_{\text{QCLI}}, I_{\text{CCI}})$, measuring (x-axis) parity-structured deviation of the correlation-order spectrum from an i.i.d. binomial baseline (QCLI) and (y-axis) the fraction of total dependence not captured by the optimal Chow--Liu tree ($I_{\text{CCI}}$). The shaded region shows a {protocol-specific} empirical IQP envelope obtained by optimizing $16$-qubit IQP circuits for high $I_{\text{QCLI}}$ and $I_{\text{CCI}}$. Datasets in the upper-right quadrant exhibit both interference-like spectral structure and beyond-pairwise dependence, forming the regime most structurally compatible with IQP-type generators. Point colors indicate data provenance: \textcolor{orange}{orange} denotes datasets produced by classical processes (e.g., simulations or classical benchmarks), whereas \textcolor{teal}{teal} denotes datasets generated by quantum processes or quantum hardware.}
 \label{fig:ccm}
\end{figure*}

Several representative datasets in Fig.~\ref{fig:ccm} have already been benchmarked with IQP generative models under the ``train on classical, deploy on quantum'' framework \citep{recio2025train}. In particular, the paper evaluates IQP generators against classical energy-based baselines Restricted Boltzmann Machine (RBM) and energy based model (EBM) on multiple binary datasets using test-set MMD diagnostics, reporting that IQP achieves the best MMD test scores in three of four experiments, and that for the large dataset D-Wave the IQP model clearly outperforms the tested classical baselines, which failed to produce convincing results.  We summarize the main qualitative outcomes in Table~\ref{tab:tocdoq-summary}.

\begin{table*}[t]
\centering
\caption{Summary of representative IQP generative modeling results reported in Ref.~\cite{recio2025train} for datasets (or similar classes) that also appear in our correlation--complexity map. We use this prior evidence to contextualize why we do not retrain IQP models for every dataset in Fig.~\ref{fig:ccm}, and to motivate focusing new experiments on the high-$I_{\mathrm{QCLI}}$/high-$I_{\mathrm{CCI}}$ turbulence regime.}
\small
\begin{ruledtabular}
\begin{tabular}{llll}
Dataset (size) & IQP setup & Baselines & Key outcome \\
\hline
BinaryBlobs (16 bits) &
\parbox[t]{3.0cm}{\raggedright Local IQP gates\\(up to 6-body),\\14892 params} &
\parbox[t]{2.2cm}{\raggedright EBM (1700 params),\\RBM, Bitflip} &
\parbox[t]{6.2cm}{\raggedright EBM achieves better likelihood than IQP (log-likelihood: EBM $-5.34$ vs.\ IQP $-6.35$); Bitflip performs poorly; MMD reveals sensitivity to mode imbalance.} \\
\hline
\parbox[t]{2.7cm}{\raggedright D-Wave annealer samples\\(484 qubits, $100\,\mu\mathrm{s}$)} &
\parbox[t]{3.0cm}{\raggedright All 1- and 2-qubit gates on 484 qubits,\\$>10^5$ params} &
\parbox[t]{2.2cm}{\raggedright RBM, EBM, Bitflip} &
\parbox[t]{6.2cm}{\raggedright No model fits perfectly, but IQP performs best among the compared models; RBM/EBM can score worse than random on test MMD; EBM fails in this setting.} \\
\hline
Binarized MNIST (784 bits) &
\parbox[t]{3.0cm}{\raggedright All-to-all 2-qubit gates,\\307{,}720 params} &
Bitflip only &
\parbox[t]{6.2cm}{\raggedright IQP captures substantial structure; covariance matches the overall pattern, albeit more weakly. Test MMD is comparable to a noisy-data reference (reported around noise level $p\approx 0.3$); Bitflip fails to learn comparable structure.} \\
\end{tabular}
\end{ruledtabular}
\label{tab:tocdoq-summary}
\end{table*}

These results provide an external consistency check for the correlation–complexity map, in regimes with substantial non-local structure (high beyond-pairwise dependence), classical training instabilities such as mode collapse/imbalance (RBM) or insufficient MCMC mixing (EBM) can dominate performance, whereas IQP training appears comparatively robust under the same compute budget. At the same time, \cite{recio2025train} explicitly cautions that their classical baselines are not state-of-the-art and that stronger classical results could plausibly be obtained with different hyperparameter choices, initialization strategies, or alternative model classes. This caveat is particularly relevant for interpreting the D-Wave point in Fig.~\ref{fig:ccm}: although \cite{recio2025train} observes IQP outperforming their tested classical baselines on this dataset, they also emphasize that stronger classical results may emerge with alternative baselines and tuning. From the correlation–complexity map perspective, the relatively low-to-moderate $I_{\text{CCI}}$ observed for certain D-Wave regimes suggests that a substantial portion of the dependency structure may remain explainable by classical pairwise/tree-based structure, making it plausible that well-tuned graphical models or other classical approaches could achieve competitive fits even when the IQP workflow performs strongly in a given benchmark. Accordingly, we treat the D-Wave result primarily as evidence that IQP can be competitive under the protocol, while the $I_{\text{CCI}}$ placement motivates a more exhaustive classical baseline search rather than an intrinsic advantage claim. 

While existing previous results help calibrate expectations for several datasets on the map, the most distinctive regime identified here, \emph{turbulence} in the high-$I_{\text{QCLI}}$/high-$I_{\text{CCI}}$ quadrant, remains untested; we therefore proceed to a dedicated IQP generative modeling study of turbulence.

\section{Generative Modeling of Turbulence via IQP Circuits}

To validate the hypothesis suggested by the correlation–complexity map, we train IQP-based generative models on turbulence snapshots and examine whether their performance reflects the dataset’s strong alignment with high-order, interference-compatible structure, where the dataset is the D2 dataset provided in \cite{khojasteh2022lagrangian}, with total 1000 turbulence snapshots with time label $t \in \{1,2,...,1000\}$.

To interface continuous-valued turbulence fields with IQP-based generative models, we convert each floating-point coordinate into a fixed-length binary representation via a simple uniform quantization-and-encoding scheme (Appendix~\ref{app:float2bit}, with some similarity of the encoding introduced in \cite{buhrman2001quantum}). Concretely, each $(x,y,z)\in\mathbb{R}^3$ sample is mapped to a bitstring in $\{0,1\}^{3N}$ by quantizing each coordinate to one of $2^N$ bins and then applying the standard $N$-bit binary encoding. In our experiments we set $N=6$, yielding an 18-bit dataset, so the IQP generator acts on $n=18$ qubits. The mapping is deterministic and efficiently invertible, allowing generated bitstrings to be decoded back into quantized turbulence coordinates for evaluation and visualization.

{
Specifically, the IQP circuit generates a distribution over quantized coordinate/value tuples rather than a full field represented as a single bitstring as in~\cite{recio2025train}. A synthetic snapshot is reconstructed by drawing many such tuples from the time-conditioned IQP distribution and aggregating them on the discretized grid. Thus, the qubit count scales with the number of bits used to represent a tuple, not with the number of grid points. This representation is appropriate for distributional field synthesis but should not be confused with direct generation of an entire high-dimensional field vector in one circuit shot.
}

To train IQP-based generative models without relying on quantum-hardware optimization, we adopt the \emph{train-on-classical, deploy-on-quantum} framework~\citep{recio2025train, coyle2020born, benedetti2019generative}. The key idea is to optimize circuit parameters on classical hardware using an objective that can be evaluated from classically tractable quantities for IQP circuits, and to use a quantum device only at deployment time for sampling. Full details are deferred to Appendix~\ref{app:tocdoq-details}.

\subsection{IQP-Based Generative Synthesis of Unseen Turbulence Snapshots}
\label{sec:latent-adaptation}
Using the train-on-classical scheme described above, we first fit an IQP distribution over quantized coordinate/value tuples for a single reference snapshot, at a fixed time index (e.g., $t=1$). 

The resulting circuit captures a compact latent representation of the {turbulence spatial} structure using only a small number of trainable parameters, far fewer than the {original number of pixels and possible combinations of pixel values} of the full turbulence field. However, a true generative model should be able to produce \emph{new} snapshots corresponding to unseen times. To achieve this, we design a latent-parameter adaptation scheme. {This enables a single shared IQP architecture, equipped with a time-dependent latent block, to define tuple distributions from which turbulence snapshots at different times are reconstructed by sampling and aggregation.}

\paragraph{Latent parameterization and snapshot adaptation.}
Let $\mathcal{D}_t=\{x^{(t)}_i\}_{i=1}^{N_t}$ denote the binarized dataset of the turbulence snapshot at time $t$, with empirical distribution $\hat p_t$. We model $\hat p_t$ using an IQP generator $q_{\theta}(x)$ induced by measuring the circuit
\begin{eqnarray}
&&U(\theta)=H^{\otimes n}\exp\!\Bigl(i\sum_{s\subseteq[n]}\theta_s Z_s\Bigr)H^{\otimes n}, \nonumber\\
&&x\sim q_{\theta}(x):=|\langle x|U(\theta)|0^n\rangle|^2.
\end{eqnarray}
We partition the parameters into a \emph{shared core} and a \emph{snapshot-specific latent block}:
\begin{equation}
\theta=\bigl(\theta_{\mathrm{core}},\,\theta_{\mathrm{lat}}\bigr),
\qquad \dim(\theta_{\mathrm{lat}})=d_{\mathrm{lat}}\ll \dim(\theta_{\mathrm{core}}).
\end{equation}

{Algorithm~\ref{alg:latent-adaptation} summarizes the latent-adaptation workflow,
including the initial core training step, sequential latent-only adaptation, and
interpolation-based synthesis at unseen times.}

\emph{Initialization at $t=1$.} For the initial snapshot (e.g.\ $t=1$), we draw $\theta_{\mathrm{lat}}^{(1)}$ from a random initialization and keep it \emph{fixed} during the first-stage training. We then fit only the core parameters $\theta_{\mathrm{core}}$ by minimizing the MMD-based loss (Appendix~\ref{app:tocdoq-details}) between the model distribution and the empirical target:
\begin{equation}
\theta_{\mathrm{core}}^{(1)} \in
\arg\min_{\theta_{\mathrm{core}}}
\mathcal{L}_{\text{MMD}}\!\bigl(q_{(\theta_{\mathrm{core}},\theta_{\mathrm{lat}}^{(1)})},\,\hat p_1\bigr),
\label{eq:fit_t1_core_only}
\end{equation}
using 500 samples from $q_{(\theta_{\mathrm{core}},\theta_{\mathrm{lat}}^{(1)})}$ to form an estimate of $\mathcal{L}_{\text{MMD}}$ at each optimization step, with a total of $3\times 10^{4}$ training steps and $N_t=5 \times 10^4$ training datapoints.

\begin{algorithm}[H]
\color{black}
\caption{Latent-parameter adaptation for turbulence synthesis}
\label{alg:latent-adaptation}
\begin{algorithmic}[1]
\State \textbf{Input:} Anchor snapshots $\{\hat p_t\}_{t\in\mathcal{T}_{\mathrm{train}}}$,
IQP model $q_{(\theta_{\mathrm{core}},\theta_{\mathrm{lat}})}$, latent dimension $d_{\mathrm{lat}}$
\State \textbf{Output:} Synthetic snapshot $\tilde y(\tau)$ at unseen time $\tau$
\State Choose an initial anchor time $t_1$ and initialize
$\theta_{\mathrm{lat}}^{(t_1)}$ randomly.
\State Train the shared core parameters by solving
\[
\theta_{\mathrm{core}}^{(1)}
\leftarrow
\arg\min_{\theta_{\mathrm{core}}}
\mathcal{L}_{\mathrm{MMD}}
\bigl(q_{(\theta_{\mathrm{core}},\theta_{\mathrm{lat}}^{(t_1)})},
\hat p_{t_1}\bigr),
\]
while keeping $\theta_{\mathrm{lat}}^{(t_1)}$ fixed.
\State Freeze $\theta_{\mathrm{core}}^{(1)}$.
\For{each later anchor time $t_j\in\mathcal{T}_{\mathrm{train}}$}
    \State Initialize $\theta_{\mathrm{lat}}^{(t_j)}$ from the previous anchor,
    $\theta_{\mathrm{lat}}^{(t_{j-1})}$.
    \State Adapt only the latent block:
    \[
    \theta_{\mathrm{lat}}^{(t_j)}
    \leftarrow
    \arg\min_{\theta_{\mathrm{lat}}}
    \mathcal{L}_{\mathrm{MMD}}
    \bigl(q_{(\theta_{\mathrm{core}}^{(1)},\theta_{\mathrm{lat}})},
    \hat p_{t_j}\bigr).
    \]
\EndFor
\State Interpolate the learned latent trajectory
$\{\theta_{\mathrm{lat}}^{(t_j)}\}$ to obtain
$\tilde\theta_{\mathrm{lat}}(\tau)$.
\State Sample bitstrings
$x\sim q_{(\theta_{\mathrm{core}}^{(1)},\tilde\theta_{\mathrm{lat}}(\tau))}$.
\State Decode sampled bitstrings into coordinate/value tuples and aggregate them
on the discretized grid to obtain $\tilde y(\tau)$.
\end{algorithmic}
\color{black}
\end{algorithm}

\emph{Sequential adaptation for $t\ge2$.} After this initial fit, we freeze the learned core parameters and reuse them across time steps in the dataset,
\begin{equation}
\theta_{\mathrm{core}} \leftarrow \theta_{\mathrm{core}}^{(1)} \quad \text{(fixed for all $t$)},
\end{equation}
and adapt only the latent block for each subsequent snapshot:
\begin{eqnarray}
&&\theta_{\mathrm{lat}}^{(t)} \in
\arg\min_{\theta_{\mathrm{lat}}}
\mathcal{L}_{\text{MMD}}\!\bigl(q_{(\theta_{\mathrm{core}}^{(1)},\theta_{\mathrm{lat}})},\,\hat p_t\bigr), \nonumber\\
&& \text{initialized at }\theta_{\mathrm{lat}}^{(t-1)},
\label{eq:fit_latent_t}
\end{eqnarray}
{with $3000$ optimization steps per adapted snapshot.}
This warm-started latent optimization exploits temporal continuity of the evolving system: adjacent snapshots correspond to similar target distributions $\hat p_t$, so a low-dimensional update in $\theta_{\mathrm{lat}}$ suffices to track the evolving state without retraining the full circuit. Repeating Eq.~(\ref{eq:fit_latent_t}) for $t=t_2,\dots,T$ yields a latent sequence $\{\theta_{\mathrm{lat}}^{(t)}\}_{t=1}^{T}$ while maintaining a single shared core $\theta_{\mathrm{core}}^{(1)}$.

\paragraph{Latent interpolation and generative synthesis of unseen times.}
The adapted latent parameters $\{\theta_{\mathrm{lat}}^{(t)}\}_{t=1}^{T}$ define a discrete trajectory in a low-dimensional subspace $\mathbb{R}^{d_{\mathrm{lat}}}$ of the IQP parameter space:
\begin{equation}
t \longmapsto \theta_{\mathrm{lat}}^{(t)} \in \mathbb{R}^{d_{\mathrm{lat}}},
\qquad \text{with shared core } \theta_{\mathrm{core}}^{(1)} \text{ fixed.}
\end{equation}
We view this trajectory as a compact representation of temporal evolution: the circuit architecture and core parameters encode the dominant spatial structure learned at $t=1$, while time variation is captured through the low-dimensional latent block.

To synthesize a snapshot at an \emph{unseen} (possibly continuous) time $\tau\in\mathbb{R}$, we construct an interpolated latent vector $\tilde\theta_{\mathrm{lat}}(\tau)$ from the learned anchors $\{\theta_{\mathrm{lat}}^{(t)}\}$ and evaluate the IQP generator at $(\theta_{\mathrm{core}}^{(1)},\tilde\theta_{\mathrm{lat}}(\tau))$. Concretely, for piecewise-linear interpolation, let $k=\lfloor \tau \rfloor$ and $\alpha=\tau-k\in[0,1]$ (assuming $1\le k < T$); we set
\begin{equation}
\tilde\theta_{\mathrm{lat}}(\tau)
=(1-\alpha)\,\theta_{\mathrm{lat}}^{(k)}+\alpha\,\theta_{\mathrm{lat}}^{(k+1)}.
\label{eq:latent_interp_linear}
\end{equation}
More generally, one may write $\tilde\theta_{\mathrm{lat}}(\tau)=\sum_{t=1}^{T} w_t(\tau)\theta_{\mathrm{lat}}^{(t)}$ with convex weights $w_t(\tau)\ge0$ and $\sum_t w_t(\tau)=1$, which also supports extrapolation by allowing non-convex weights when $\tau\notin[1,T]$.

Given $\tilde\theta_{\mathrm{lat}}(\tau)$, we sample bitstrings
\begin{equation}
x \sim q_{(\theta_{\mathrm{core}}^{(1)},\,\tilde\theta_{\mathrm{lat}}(\tau))}(x),
\end{equation}
and decode each sample back to a floating-point turbulence field via the inverse of our bitstring encoding (Appendix~\ref{app:float2bit}), yielding a synthetic snapshot $\tilde y(\tau)$ in the original physical domain. In practice, at each generated time $\tau$ we draw a fixed number of samples per synthesis call (e.g., $10^5$ samples) from the IQP generator and aggregate them according to the same probability-aggregation/decoding pipeline used during training. This defines a continuous-time generator
\begin{equation}
\tau \ \mapsto\  q_{(\theta_{\mathrm{core}}^{(1)},\,\tilde\theta_{\mathrm{lat}}(\tau))} \ \mapsto\  \tilde y(\tau),
\end{equation}
which produces temporally coherent turbulence fields while maintaining a fixed and compact IQP circuit structure. The overall workflow is illustrated in Fig.~\ref{fig:latent-flow}. The learned latent trajectory $\{\theta_{\mathrm{lat}}^{(t)}\}_{t=1}^{T}$ (and the corresponding interpolation $\tilde\theta_{\mathrm{lat}}(\tau)$) can be visualized by plotting each latent coordinate $\bigl(\theta_{\mathrm{lat}}^{(t)}\bigr)_j$ against the time index $t$, as shown in Fig.~\ref{fig:latent-traj}. Detailed experimental settings and hyperparameters for this visualization are provided in the next section.

\begin{figure*}[ht]
\centering
\includegraphics[width=0.9\linewidth]{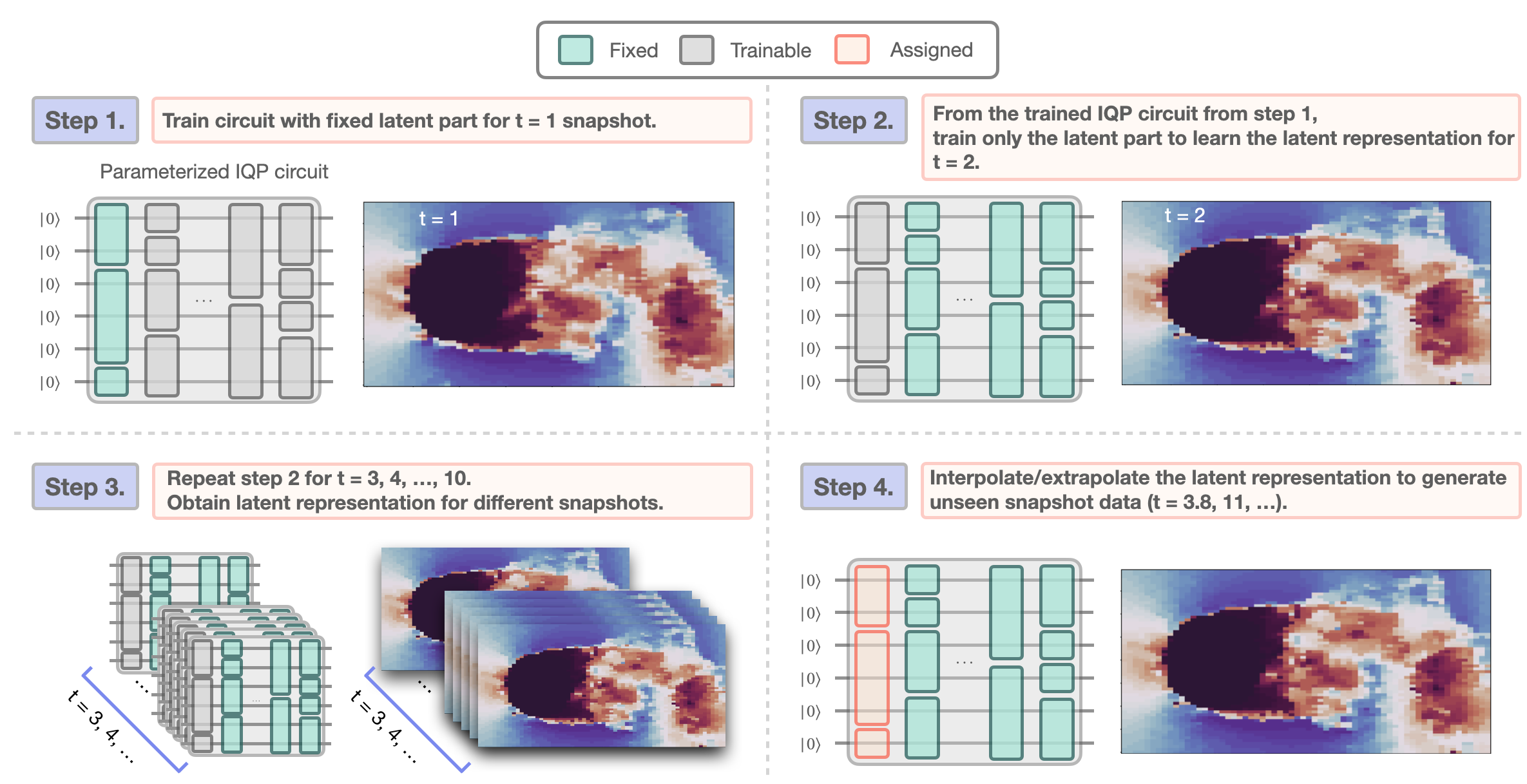}

\caption{Latent-parameter adaptation and interpolation for IQP-based generative synthesis of turbulence snapshots.
Step 1: Fit a single IQP generator to the reference snapshot at $t=1$ by optimizing the shared core parameters $\theta_{\mathrm{core}}$ while holding the latent block $\theta_{\mathrm{lat}}^{(1)}$ fixed at its random initialization. Step 2: Freeze $\theta_{\mathrm{core}}^{(1)}$ and adapt only the low-dimensional latent parameters $\theta_{\mathrm{lat}}$ to match the next snapshot $(t=2)$, yielding $\theta_{\mathrm{lat}}^{(2)}$. Step 3: Repeat the latent-only adaptation sequentially for $t=3,4,\ldots,T$ to obtain a latent trajectory $\{\theta_{\mathrm{lat}}^{(t)}\}_{t=1}^{T}$ while reusing the same fixed core circuit. Step 4: Generate snapshots at unseen times $\tau$ by assigning latent parameters via interpolation/extrapolation $\tilde{\theta}_{\mathrm{lat}}(\tau)$ (e.g., piecewise-linear interpolation between neighboring $\theta_{\mathrm{lat}}^{(t)})$, and sampling from $q_{(\theta_{\mathrm{core}}^{(1)},\tilde{\theta}_{\mathrm{lat}}(\tau))}$ followed by decoding to the physical turbulence field. The time indices $(e.g., t=2,3,\ldots,10)$ and turbulence snapshots shown are for illustration only; our experiments use a different subset of observed snapshots, with the same adaptation and interpolation procedure applied unchanged.}
 \label{fig:latent-flow}
\end{figure*}

\begin{figure}[ht]
\centering
\includegraphics[width=\linewidth]{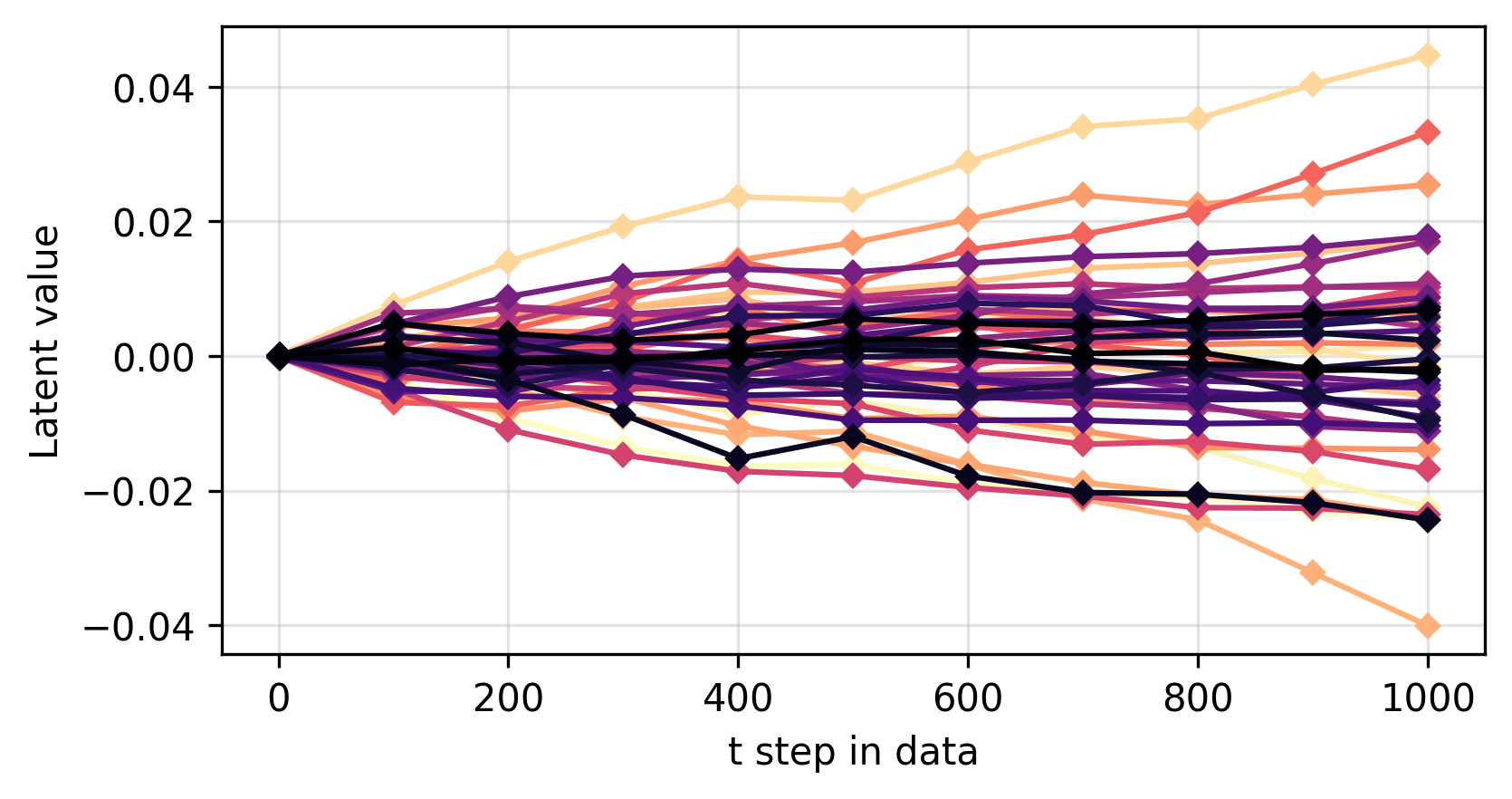}
\caption{Learned latent-parameter trajectory across snapshot adaptation steps.
Each curve shows one latent coordinate $\bigl(\theta_{\mathrm{lat}}^{(t)}\bigr)_j$ as a function of the snapshot index $t\in\{1,100,200,300,400,500,600,700,800,900,1000\}$, obtained by sequential latent-only adaptation while keeping the shared core parameters $\theta_{\mathrm{core}}^{(1)}$ fixed. All $d_{\mathrm{lat}}=50$ latent coordinates are plotted. The gradually varying evolution across $t$ supports viewing $\{\theta_{\mathrm{lat}}^{(t)}\}$ as a low-dimensional trajectory and this trajectory motivates interpolation/extrapolation $\tilde{\theta}_{\mathrm{lat}}(\tau)$ to synthesize turbulence snapshots at unseen times.}
 \label{fig:latent-traj}
\end{figure}

\subsection{Comparative Evaluation with Classical Generative Models}

To benchmark the generative capability of the proposed IQP-based method, we compare against two classical baselines: a RBM \citep{fischer2012introduction} and a deep convolutional GAN (DCGAN) \citep{radford2015unsupervised}. Each model is trained on turbulence data using the representation most natural for its architecture, while keeping the training data and evaluation protocol aligned to enable a meaningful comparison.

{
The baselines used here are intended to test two contrasting regimes: RBM as a bitstring-domain energy model trained on the same representation as IQP, and DCGAN as a continuous-domain convolutional generator with a strong spatial inductive bias. They are not exhaustive. In particular, diffusion models, autoregressive models, tensor-network distributions, coordinate-based neural fields, and reduced-order physical models could possibly provide stronger baselines. We therefore interpret the results as evidence of low-data competitiveness against the tested baselines, not as a comprehensive classical comparison.
}

\paragraph{Training protocol for IQP and RBM models.}
Following the latent-adaptation procedure in Sec.~\ref{sec:latent-adaptation}, the IQP circuit is trained on $11$ turbulence snapshots at time indices
\[
t \in \{1,\;100,\;200,\;300,\;400,\;500,\;600,\;700,\;800,\;900,\;1000\}.
\]
These snapshots serve as anchor points for learning a discrete latent trajectory $\{\theta_{\mathrm{lat}}^{(t)}\}$. After obtaining latent parameters at these anchor times, we uniformly sample $200$ additional time indices from $[1,1000]$ (excluding the training points) and generate the corresponding turbulence fields by interpolating the latent parameters, yielding a synthesized test set of $200$ snapshots.

For the RBM baseline, we use the \emph{same} $11$ snapshots as training data and train directly in the bitstring domain, i.e., on the identical binarized representation used for IQP training. This ensures that IQP and RBM observe the same information content and the same training cardinality.

\paragraph{Setting of latent dimension and parameter partition.}
For the IQP model, we set the latent dimension to $d_{\mathrm{lat}}=50$, i.e., $\theta_{\mathrm{lat}}\in\mathbb{R}^{50}$, and use the remaining parameters as the shared core $\theta_{\mathrm{core}}$. We define $\theta_{\mathrm{lat}}$ as the first $d_{\mathrm{lat}}$ elements of the canonical IQP parameter vector obtained by ordering the commuting Pauli strings by increasing Hamming weight,
\[
\{\emptyset\},\ \{0\},\ \{1\},\ldots,\ \{n-1\},\ \{0,1\},\ \{0,2\},\ldots,
\]
so that low-order (one- and two-body) terms appear before higher-order terms. Under this convention, the latent block is dominated by low-order commuting interactions, providing a compact set of degrees of freedom that can be adjusted across time, while the remaining parameters form $\theta_{\mathrm{core}}$ and are held fixed during snapshot adaptation.

For the RBM, we use an analogous \emph{latent dimension} for temporal interpolation by selecting a $d_{\mathrm{lat}}$-dimensional parameter subvector to vary across time while keeping the remaining RBM parameters fixed. Specifically, we take the first $d_{\mathrm{lat}}=50$ elements of a fixed, flattened parameter ordering (weights and biases) as the RBM latent block and treat the resulting latent trajectory identically to the IQP case when generating snapshots at unseen times. This choice does not impose an architectural constraint on the RBM; rather, it enforces a matched low-dimensional interpolation mechanism so that differences in performance reflect model inductive bias rather than differences in adaptation degrees of freedom.

\paragraph{Training protocol for DCGAN (continuous-domain model).}
Unlike RBM and IQP, DCGAN operates in the continuous image domain and leverages convolutional inductive biases for spatial structure. Accordingly, we train DCGAN directly on the \emph{quantized 2D turbulence fields} (rather than the bitstring representation), so that the GAN is not artificially constrained by our encoding pipeline.

We report two DCGAN regimes to disentangle \emph{model/architecture effects} from \emph{data availability effects}. 

(i) \textbf{Data-rich DCGAN:} we train DCGAN on $100$ turbulence snapshots (approximately $10\times$ the training set size used for IQP/RBM) to reflect standard GAN practice and to assess its performance given sufficient data. 

(ii) \textbf{Data-matched DCGAN:} we additionally train DCGAN on the same $11$ snapshots used by IQP/RBM to evaluate robustness under the low-data regime relevant to our setting. 
In both cases, DCGAN is trained on the same underlying turbulence distribution, and evaluation is performed on the same held-out protocol described below.

We choose the DCGAN latent dimension $128$ to be comparable to the effective latent capacity of the IQP model (in the sense that both define a relatively low-dimensional control variable used to generate diverse snapshots), while allowing the convolutional generator to represent spatial correlations in its native domain.

\paragraph{Quantitative and qualitative comparison.}
Tables~\ref{tab:turbulence-samples-100} presents qualitative samples together with two quantitative metrics: (i) the PDF--JS divergence (lower is better), and (ii) a feature-space $\text{MMD}$ computed using a random Conv2D encoder (lower is better). All metrics are evaluated by comparing $200$ ground truth turbulence snapshots against $200$ samples generated from each model. 
{PDF--JS probes the marginal distribution of field intensities and is insensitive to spatial arrangement. The random Conv2D feature MMD probes higher-order spatial statistics through a fixed nonlinear embedding, but it is not a substitute for physics-specific turbulence diagnostics such as energy spectra or structure functions. We therefore use these metrics as distributional proxies and leave a full physics-level turbulence validation to future work. 
}Details of these two score are shown in the Appendix~\ref{app:evaluation_metrics}.  

\begin{table*}[t]
  \centering
  \small
  \setlength{\tabcolsep}{6pt}
  \renewcommand{\arraystretch}{1.25}
  \begin{tabular}{l c c c}
    \toprule
    Model (Params.)(\# TD.) & Samples & PDF--JS $\downarrow$ & MMD$^2$ (Conv2D) $\downarrow$ \\
    \midrule

    \textbf{Real (---)(---)} &
    \raisebox{-0.5\height}{\includegraphics[width=0.4\linewidth]{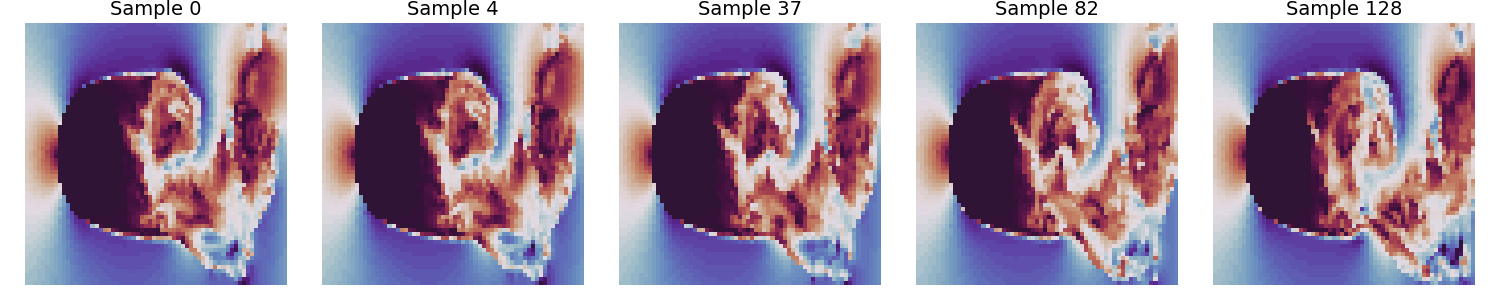}} &
    -- & -- \\

    \textbf{RBM (12k)(11)} &
    \raisebox{-0.5\height}{\includegraphics[width=0.4\linewidth]{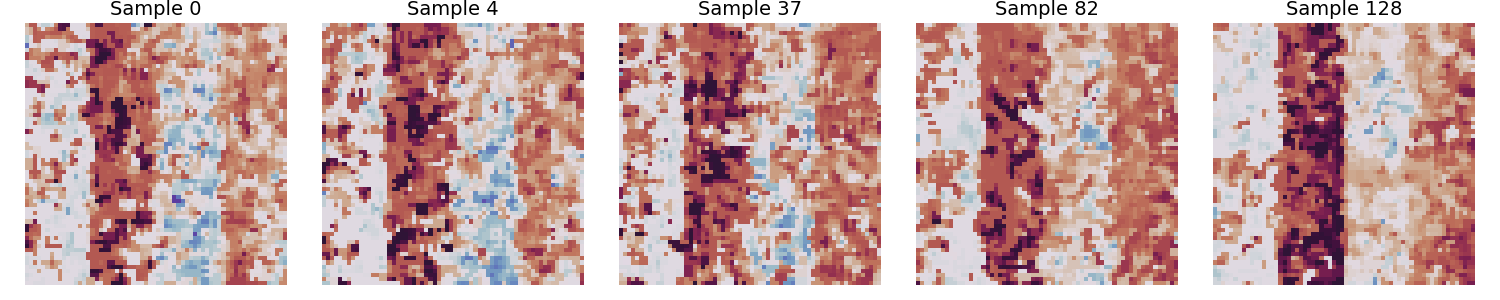}} &
    0.168770 & 1.472481 \\

    \textbf{IQP (31k)(11)} &
    \raisebox{-0.5\height}{\includegraphics[width=0.4\linewidth]{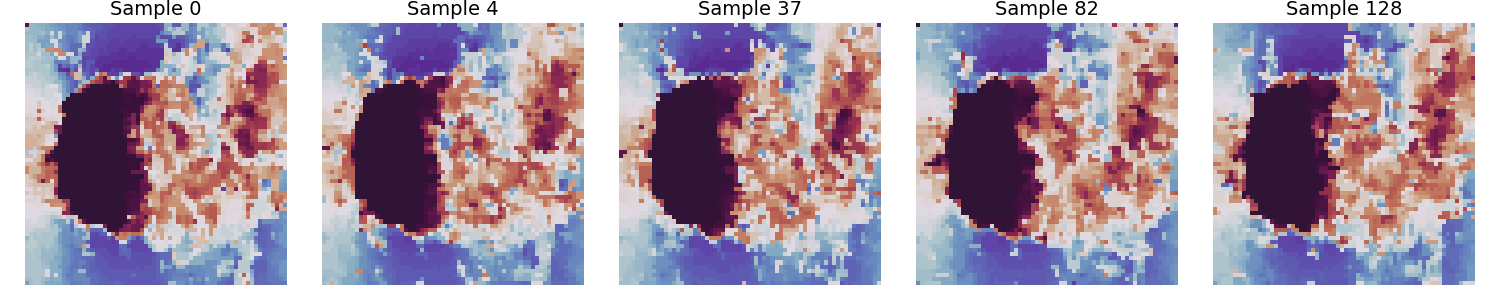}} &
    0.023784 & 0.952420 \\

    \textbf{IQP (63k)(11)} &
    \raisebox{-0.5\height}{\includegraphics[width=0.4\linewidth]{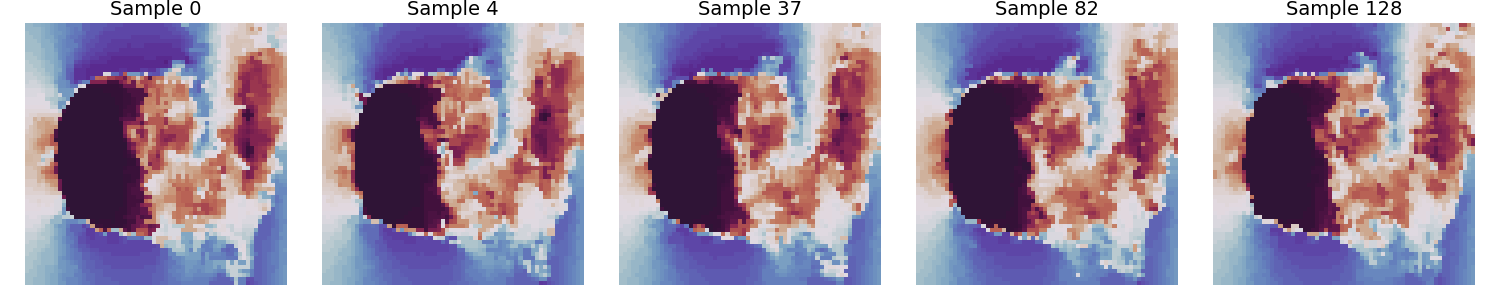}} &
    \textbf{0.022084}$^\dagger$ & \textbf{0.942267}$^\dagger$ \\

    \textbf{GAN (51k)(100)} &
    \raisebox{-0.5\height}{\includegraphics[width=0.4\linewidth]{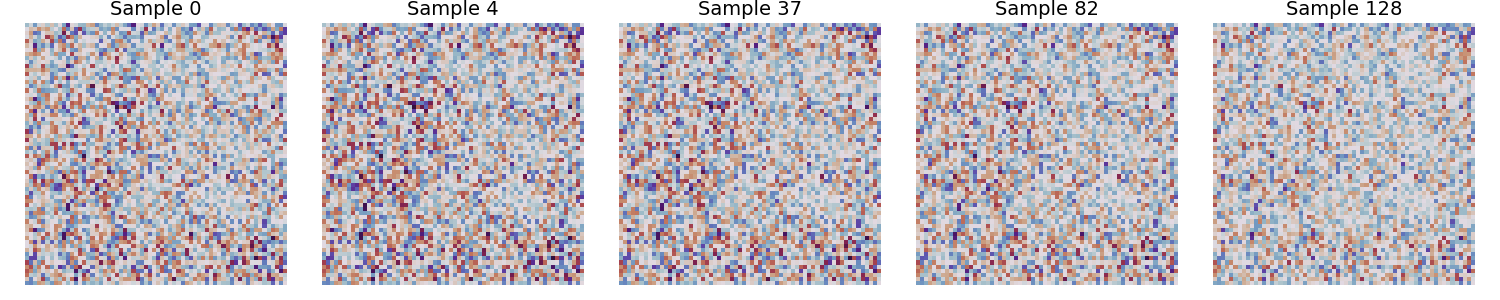}} &
    0.135584 & 1.327760 \\

    \textbf{GAN (0.2M)(100)} &
    \raisebox{-0.5\height}{\includegraphics[width=0.4\linewidth]{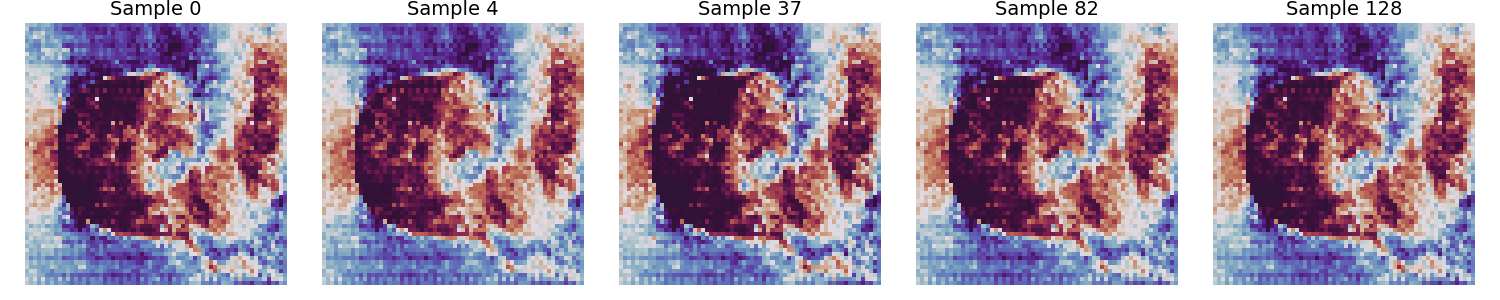}} &
    0.027695 & \textbf{0.895912} \\

    \textbf{GAN (1.2M)(100)} &
    \raisebox{-0.5\height}{\includegraphics[width=0.4\linewidth]{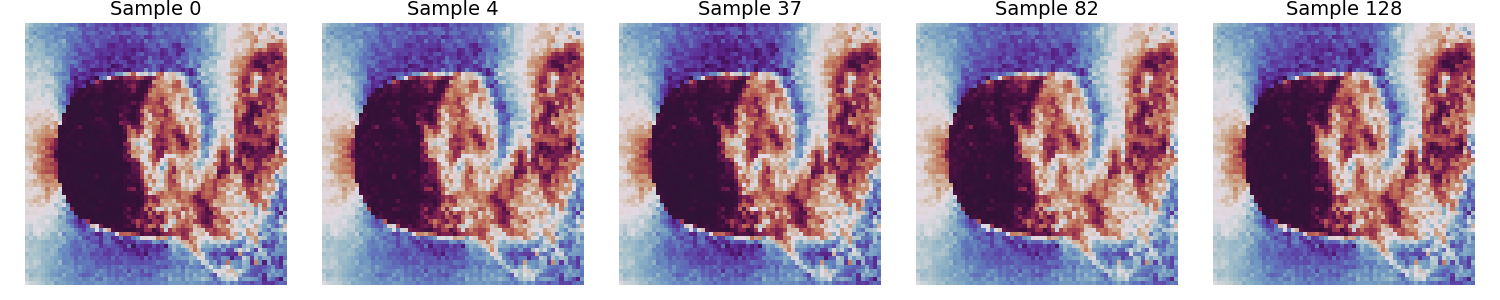}} &
    \textbf{0.012204} & 0.964874 \\

    \textbf{GAN (51k)(11)} &
    \raisebox{-0.5\height}{\includegraphics[width=0.4\linewidth]{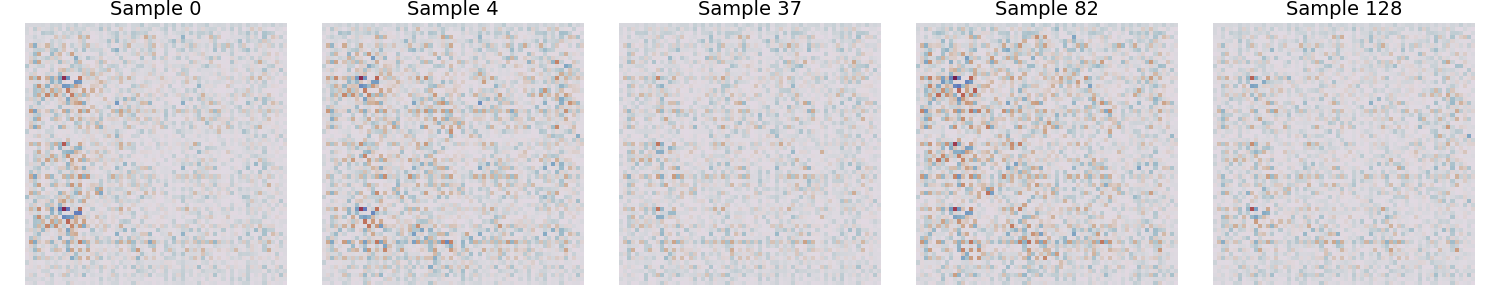}} &
    0.366758 & 1.397969 \\

    \textbf{GAN (0.2M)(11)} &
    \raisebox{-0.5\height}{\includegraphics[width=0.4\linewidth]{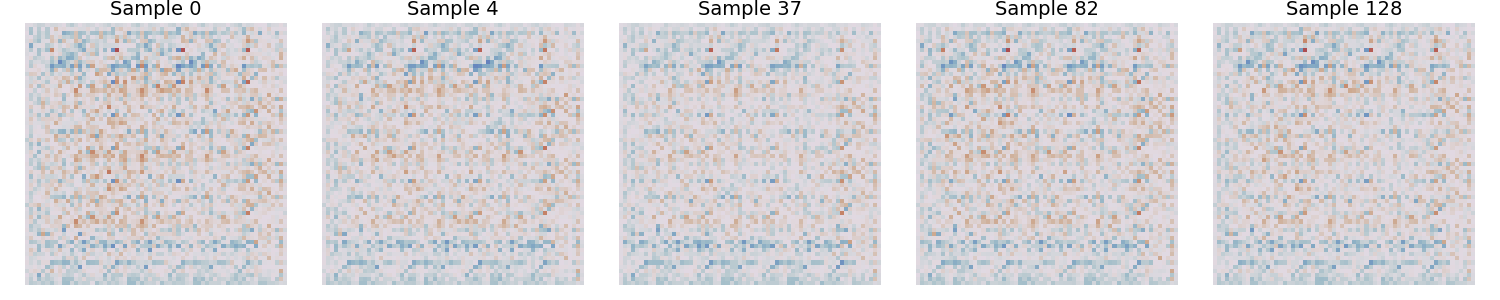}} &
    0.328273 & {1.407441} \\

    \textbf{GAN (1.2M)(11)} &
    \raisebox{-0.5\height}{\includegraphics[width=0.4\linewidth]{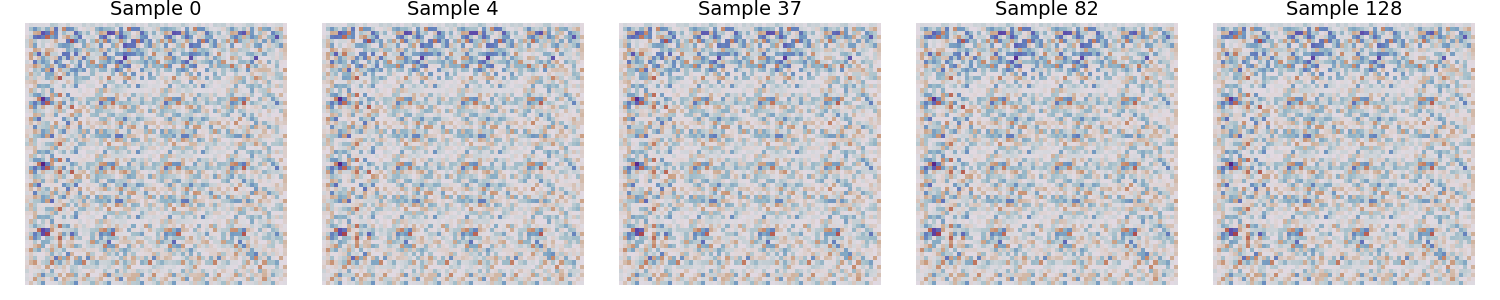}} &
    {0.244852} & 1.446292 \\

    \bottomrule
  \end{tabular}
  \caption{Qualitative samples (center) with corresponding PDF--JS and feature-space MMD$^2$ scores (right). Lower scores indicate better alignment with the real turbulence distribution. Models are shown with their parameter counts in parentheses.}
  \label{tab:turbulence-samples-100}
\end{table*}

\textbf{Data-matched regime (11 training snapshots).}
When DCGAN is trained on only $11$ snapshots, it fails to learn a meaningful generative distribution, samples exhibit severe artifacts consistent with under-training and/or instability, and both PDF--JS and feature $\text{MMD}$ degrade substantially relative to IQP. In contrast, the IQP models (31k and 63k parameters) remain stable and produce visually coherent turbulence-like fields with markedly better PDF--JS and $\text{MMD}$ than both RBM and DCGAN in this low-data setting. This supports the interpretation that IQP inductive biases are well-aligned with turbulence structure and can be leveraged in a sample-limited regime.

\textbf{Data-rich regime (100 training snapshots).}
With $100$ training snapshots, DCGAN improves significantly, and high-capacity variants can match or surpass IQP on one or both metrics. Importantly, this gain requires substantially more training data (about $10\times$) and, for the strongest variants, substantially more parameters (up to the $1.2$M setting). By contrast, the IQP models achieve competitive performance while being trained on only $11$ snapshots and using a fixed, compact circuit with latent adaptation. 

These results highlight a regime-dependent trade-off, DCGAN can perform very well when provided with sufficient data and capacity, but is brittle in the extremely low-data regime, whereas the IQP latent-adaptation scheme remains effective and sample-efficient on turbulence. This empirical behavior is consistent with turbulence occupying a comparatively IQP-compatible region (high-$I_{\text{QCLI}}$/high-$I_{\text{CCI}}$) of the correlation--complexity map.

\paragraph{Data scaling of DCGAN relative to IQP baselines.}
To quantify the data requirement for DCGAN to reach IQP-level performance, we perform a data-scaling study in which DCGAN is trained on varying numbers of turbulence snapshots, while evaluation is kept fixed. Specifically, we train DCGAN with three capacities (51k, 0.2M, and 1.2M parameters) using $N_{\text{train}}\in\{10,30,50,70,90,100\}$ snapshots and evaluate each model using the same protocol as above (PDF–JS and Conv2D feature-space MMD on a held-out set). Fig.~\ref{fig:gan_data_scaling} compares these results to the IQP baselines (horizontal dashed lines), where the IQP curves correspond to the $N_{\text{train}}=11$ latent-adaptation setting. As expected, DCGAN performance improves with additional training data and model capacity, but the trend is not strictly monotone, especially for the intermediate-capacity GAN, reflecting the well-known instability of GAN training in the low-data regime \citep{karras2020training, zhao2020differentiable}. Importantly, matching the IQP’s PDF–JS level typically requires on the order of $10^2$ training snapshots for the high-capacity DCGAN (1.2M), whereas small DCGAN variants remain substantially worse even at $N_{\text{train}}=100$. On the feature-space MMD metric, DCGAN also narrows the gap as data increases, approaching the IQP reference only at the highest data budgets. Overall, this scaling behavior reinforces the training sample-efficiency advantage of the IQP latent-adaptation scheme on turbulence, where IQP achieves competitive distributional alignment using only 11 training snapshots, while DCGAN generally requires an order-of-magnitude more data to reach comparable scores.

\begin{figure*}[ht]
\centering
\includegraphics[width=0.8\linewidth]{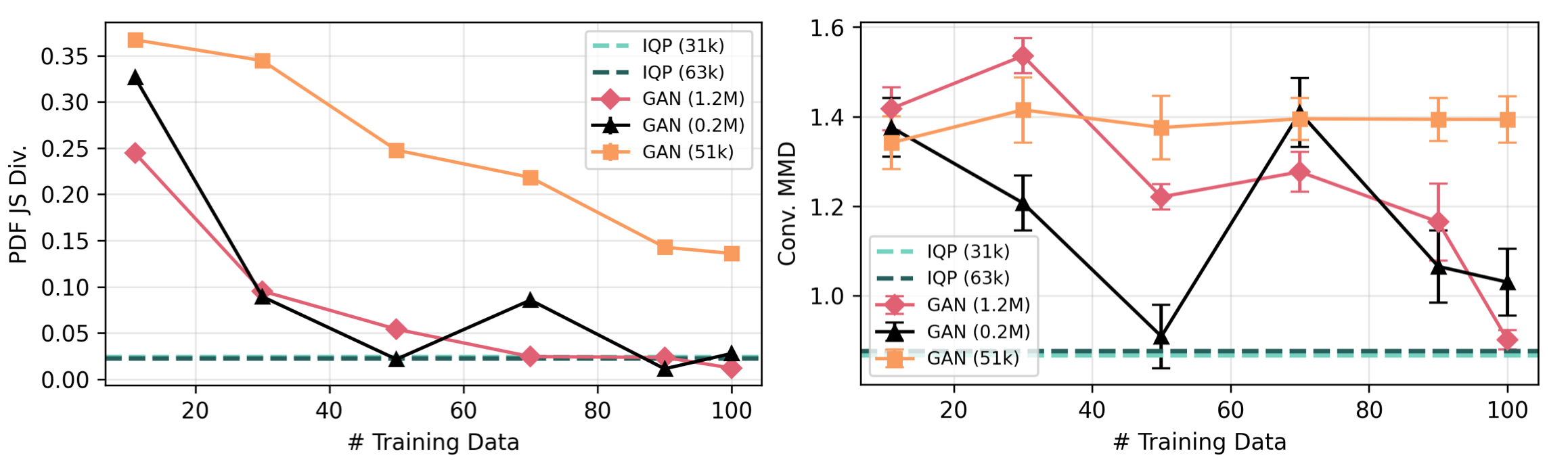}
\caption{DCGAN data-scaling versus IQP baselines on turbulence.
Left: PDF--JS divergence; right: Conv2D feature-space MMD (lower is better), both evaluated by comparing 200 real snapshots to 200 generated samples. Solid lines show DCGAN performance as a function of the number of training snapshots $N_{\text{train}}\in\{10,30,50,70,90,100\}$ for three model sizes (51k, 0.2M, 1.2M parameters); error bars indicate variability across independent runs. Horizontal dashed lines denote the IQP reference performance under the latent-adaptation setting trained on 11 snapshots (the same setting used in Tables~\ref{tab:turbulence-samples-100}).}

 \label{fig:gan_data_scaling}
\end{figure*}

\subsection{On Qubit Count Efficiency and Practicality of Quantum Generative Methods}

A central practical challenge for generative quantum models is the large number of qubits required to represent high-dimensional datasets. In the original formulation with \cite{recio2025train, benedetti2019generative}, the quantum circuit directly models the bitstring representation of the data. For datasets such as MNIST, where each image consists of $28\times 28 = 784$ pixels, this implies that inference requires a $784$-qubit quantum processor. While this is already beyond the scale of near-term quantum hardware, the limitation becomes even more severe for scientific datasets.

In our turbulence experiment, each three-dimensional snapshot has spatial resolution $64\times 64\times 64 = 262{,}144$ grid points. Under the original formulation, faithfully learning or generating such a sample would therefore require a quantum circuit acting on $262{,}144$ qubits, placing this task entirely outside the reach of foreseeable quantum devices. Even though this allows the same trained circuit to generate new unseen samples without retraining, the qubit requirement remains fundamentally tied to the native dimensionality of the data.

In contrast, our latent interpolation approach replaces the full high-dimensional representation with a compact \emph{18-bit} encoding obtained via the float-to-bitstring mapping described earlier. Each turbulence sample is mapped to a bitstring set of length $18$, and consequently our IQP model operates on an $18$-qubit circuit. Crucially, the temporal variability of the turbulence field is captured not by increasing the number of qubits, but through the latent-parameter interpolation scheme introduced in Section~\ref{sec:latent-adaptation}. A single 18-qubit IQP circuit, equipped with a time-informed latent parameter block, can generate an entire family of turbulence snapshots across time.

Thus, while the \cite{recio2025train} framework scales in qubit count proportional to the raw dimensionality of the dataset, our float-to-bitstring mapping and latent interpolation methods decouple generative expressivity from input dimensionality. This enables the synthesis of complex, high-resolution scientific fields using quantum circuits of fixed, low qubit count, a key step toward making generative quantum modeling practical on near-term hardware.

{Other existing quantum generative approaches span several modeling philosophies. Energy-based formulations such as Quantum Boltzmann Machines (QBMs) \citep{amin2018quantum, benedetti2017quantum} aim to represent complex distributions via quantum-native energy landscapes (or quantum-assisted sampling), inheriting both the expressive promise and the training difficulty of Boltzmann-style learning.
In contrast, quantum generative adversarial learning and its variants \citep{zoufal2019quantum, lloyd2018quantum}, replace explicit likelihood/partition functions with discriminator, driven training signals, but can be sensitive to optimization instability and data-hungriness, especially in high-dimensional regimes.}

\section{Impact of the Proposed Indicators and Methods}
\label{sec:impact}

A recurring bottleneck in the pursuit of \emph{useful} (and potentially advantageous) generative QML is not only \emph{how} to train classically hard quantum generators at scale, but \emph{when} doing so is structurally well-motivated for a given real-world dataset. The contributions of this work are intended to operationalize that second question: given a dataset, we provide two computable indicators, $I_{\text{QCLI}}$ and $I_{\text{CCI}}$, that jointly estimate (i) \emph{architecture--data alignment} with IQP-type parity/interference structure and (ii) \emph{beyond-pairwise} dependence in the classical correlation structure. Together, they turn the search for generative quantum utility from an ad hoc benchmark exercise into a more principled, data-driven screening procedure.

From a practitioner’s perspective, the correlation--complexity map can be viewed as a lightweight ``pre-check'' before committing to expensive model development. For a candidate dataset (or dataset transform), one can:
(i) compute $I_{\text{QCLI}}$ to quantify whether the data exhibits parity-structured spectral mass that is naturally expressible by an IQP generator;
(ii) compute $I_{\text{CCI}}$ to quantify whether the dataset’s dependence structure substantially exceeds what is captured by pairwise/tree-structured models; and
(iii) compare the resulting point to empirical IQP-occupied regions in the $(I_{\text{QCLI}}, I_{\text{CCI}})$ plane.
This workflow supports two actionable outcomes: it can \emph{de-prioritize} datasets that are simultaneously low-$I_{\text{QCLI}}$ and low-$I_{\text{CCI}}$ (where classical low-order models are typically sufficient), and it can \emph{prioritize} datasets that lie in high-$I_{\text{QCLI}}$/high-$I_{\text{CCI}}$ regimes where IQP inductive biases are expected to be most relevant.

A key value of $I_{\text{QCLI}}$ is that it is not merely a descriptive statistic, it is explicitly designed to probe the same representational mechanism that appears in the support-mismatch analysis (via Walsh--Fourier/parity support). While {the support-mismatch analysis} provides a lower-bound-style statement (and therefore does not directly certify achievable training performance under finite optimization budgets), the empirical probing results show that $I_{\text{QCLI}}$ meaningfully predicts \emph{regime structure} in the observed MMD outcomes. In particular, the envelope trends seen under controlled experiments are consistent with the interpretation that higher $I_{\text{QCLI}}$ constrains the extent of spectral leakage outside the learner’s accessible support, thereby limiting worst-case approximation mismatch even when optimization noise and sampling error are present.

The $I_{\text{CCI}}$ axis is included to capture an orthogonal but practically important aspect, even if a dataset is IQP-compatible in the parity/spectral sense, it may still be \emph{classically easy} if its dependence structure is largely pairwise (or well-approximated by a tree model). By introducing $I_{\text{CCI}}$ and empirically observing its coupling with high-$I_{\text{QCLI}}$ IQP outputs, the map provides a more nuanced structural target: regimes that are simultaneously IQP-aligned \emph{and} correlation-complex. Conceptually, this clarifies why ``quantum-likeness'' alone is not the right proxy for utility, the most compelling targets for heuristic quantum utility are those whose structure combines IQP-compatible interference patterns with substantial beyond-pairwise dependence.

A second practical impact of this work is methodological, the float-to-bitstring mapping and the latent-parameter adaptation mechanism broaden the scope of \cite{recio2025train}-style workflows beyond datasets that are natively binary and low-dimensional. In particular, the turbulence study illustrates a concrete use case where (i) the raw data are high-dimensional continuous fields, (ii) direct bitstring modeling would imply an infeasible qubit count, yet (iii) compact encodings and latent adaptation enable generation with a fixed, small-$n$ IQP circuit. This supports a pragmatic interpretation of ``utility'' in the near term, not a worst-case complexity-theoretic separation, but a demonstration that classically hard circuit families can serve as effective inductive-bias models for nontrivial scientific data \emph{under realistic resource constraints}.

\color{black}
\paragraph*{Limitations}
\label{sec:limitations}

We emphasize that the proposed indicators are \emph{diagnostic} rather than \emph{certifying}. High $(I_{\mathrm{QCLI}}, I_{\mathrm{CCI}})$ suggests a promising regime for IQP inductive-bias matching, but it does not certify quantum advantage, prove that high-$I_{\mathrm{QCLI}}$ datasets are always easy for IQP models to learn, or rule out the existence of stronger classical baselines, improved training heuristics, and alternative data representations that could shift a dataset's effective position on the map. Likewise, low $I_{\mathrm{QCLI}}$ does not imply that IQP models cannot learn; it indicates only that the specific parity-structured support-mismatch mechanism emphasized here is unlikely to be the dominant explanation for success.

The turbulence experiments should therefore be interpreted as evidence of low-data competitiveness against the tested baselines, not as a hardware sampling advantage or as a claim that turbulence is intrinsically classically hard to sample. In this sense, the role of the correlation--complexity framework is to make the dataset-search problem more systematic and falsifiable: it provides measurable hypotheses about \emph{where} IQP generators may be competitive, which can then be tested against stronger baselines, alternative representations, and eventually hardware-level sampling constraints.
\color{black}

\section{Conclusion}
\label{sec:conclusion}

This work takes a step toward making the search for generative quantum \emph{utility} more systematic, rather than treating dataset choice as a collection of isolated benchmarks, we advocate diagnosing dataset structure through indicator-driven screening and then testing the resulting alignment hypotheses with scalable IQP training protocols. In particular, the turbulence case study illustrates that, under a compact encoding and latent-parameter adaptation, a fixed low-qubit IQP generator can achieve competitive distributional alignment in an extreme low-data regime where standard GAN training becomes unstable.

Several directions remain open. A first priority is a broader and more rigorous baseline suite for datasets that appear IQP-compatible on the correlation--complexity map, including stronger classical graphical models, diffusion/score-based generators, and representation-learning pipelines, to better describe when IQP inductive bias provides unique value in practice, beyond complexity-theoretic hardness arguments (e.g., PH-collapse implications), and to identify the concrete data and resource regimes where it yields measurable gains. 

Second, the indicators themselves can be extended and stress-tested: e.g., evaluating sensitivity to quantization/binarization choices, designing representation-search procedures that explicitly optimize $(I_{\text{QCLI}}, I_{\text{CCI}})$, and developing conditional or multiscale variants tailored to structured scientific fields. 

Third, the map can be extended to include a tensor-network-oriented notion of classical tractability, since tensor-network simulation is a frequent and strong comparator for quantum generative workflows. A tensor-network axis would refine the diagnostic by distinguishing ``hard for our current baselines” from “hard for the best-known structured classical simulators'', and would guide more principled baseline design in future studies.

Finally, it will be important to close the loop toward deployment by incorporating hardware-relevant constraints (finite shots, noise, connectivity) and validating whether the predicted IQP-compatible regimes persist when 
 is performed on real quantum devices.

\begin{acknowledgments}
We are grateful to Jacob Swain, Marco Ballarin, Stephen Clark, and Harry Buhrman for helpful discussions.
\end{acknowledgments}

\bibliographystyle{apsrev4-2}
\bibliography{ref}
% \bibliography{iclr2025_conference, bib/qml, bib/qt}

\clearpage
\onecolumngrid
\appendix

\section{Derivation detail}
\label{sec:ddt_1}
The measurement outcome of IQP circuit in the computational ($Z$) basis is
\begin{eqnarray}
p(z)
&=& \bigl|\langle z|U|0\rangle^{\otimes n}\bigr|^2 \\
&=& \bigl|\langle z| H^{\otimes n}
\Bigl(
    \exp\!\bigl(i\!\!\sum_{s\subseteq [n]} \theta_s Z_s\bigr)
\Bigr)
H^{\otimes n} |0\rangle^{\otimes n}\bigr|^2   \\
&=& \left|\left( 2^{-n/2} \sum_{y \in \{0,1 \}} (-1)^{z\cdot y} \langle y|  \right) \Bigl(
    \exp\!\bigl(i\!\!\sum_{s\subseteq [n]} \theta_s Z_s\bigr)
\Bigr) \left(2^{-n/2} \sum_{x\in \{0,1 \} }|x\rangle \right) \right|^2 \\
&=& \left|2^{-n} \sum_{y,x}(-1)^{z \cdot y} \delta_{y,x} \exp(i \sum_{s \subseteq[n]} \theta_s \chi_s(x) ) \right|^2\\
&=&
\left|
    \frac{1}{2^n}
    \sum_{x\in\{0,1\}^n}
        (-1)^{z\cdot x}
        e^{i\sum_s \theta_s \chi_s(x)}
\right|^2,
\end{eqnarray}
\subsection{JS divergence approximates the \texorpdfstring{$\ell_2$}{} distance.}
\label{app:JS_distance}
Let $b=(b_0,\ldots,b_n)$ be the binomial reference on orders, and let
\begin{equation}
m = b + \delta,\qquad \sum_{k=0}^n \delta_k = 0,\qquad |\delta_k|\ll b_k.
\end{equation}
Define $M=\tfrac12(m+b)=b+\tfrac12\delta$.
Write the (base-2) entropy as $H_2(p) = -\sum_k p_k \log_2 p_k$.
Then the Jensen–Shannon divergence (in bits) is
\begin{equation}
D_{\mathrm{JS}}(m\Vert b)
= H_2(M) - \tfrac12 H_2(m) - \tfrac12 H_2(b).
\end{equation}
We expand $H_2(\cdot)$ to second order in $\delta$. It’s convenient to do the algebra with natural logs and divide by $\ln 2$ at the end:
\begin{equation}
H(p) = -\sum_k p_k \ln p_k, \qquad H_2(p)=\frac{H(p)}{\ln 2}.
\end{equation}
Use $\ln(b_k+\varepsilon)=\ln b_k + \frac{\varepsilon}{b_k} - \frac{\varepsilon^2}{2 b_k^2} + O(\varepsilon^3$). Then
\begin{eqnarray}
\begin{aligned}
H(m)
&= -\sum_k (b_k+\delta_k)\Bigl(\ln b_k + \frac{\delta_k}{b_k} - \frac{\delta_k^2}{2 b_k^2}\Bigr) + O(\|\delta\|^3)\\
&= -\sum_k b_k\ln b_k \;-\; \sum_k \delta_k \ln b_k \;-\; \sum_k \frac{\delta_k^2}{b_k} \;+\;\frac12\sum_k \frac{\delta_k^2}{b_k} \;+\; O(\|\delta\|^3)\\
&= H(b)\;-\;\sum_k \delta_k \ln b_k \;-\;\frac12\sum_k \frac{\delta_k^2}{b_k} \;+\; O(\|\delta\|^3).
\end{aligned}
\end{eqnarray}
Similarly, with $M=b+\tfrac12\delta$,
\begin{eqnarray}
\begin{aligned}
H(M)
&= -\sum_k \Bigl(b_k+\frac{\delta_k}{2}\Bigr)\Bigl(\ln b_k + \frac{\delta_k}{2 b_k} - \frac{\delta_k^2}{8 b_k^2}\Bigr) + O(\|\delta\|^3)\\
&= H(b)\;-\;\frac12\sum_k \delta_k \ln b_k \;-\;\frac18\sum_k \frac{\delta_k^2}{b_k} \;+\; O(\|\delta\|^3).
\end{aligned}
\end{eqnarray}
Plug into the JS definition (with natural logs first):
\begin{eqnarray}
\begin{aligned}
D_{\mathrm{JS}}^{(\mathrm{nat})}(m\Vert b)
&= H(M) - \tfrac12 H(m) - \tfrac12 H(b)\\
&= \Bigl[H(b) - \tfrac12\sum_k \delta_k \ln b_k - \tfrac18\sum_k \frac{\delta_k^2}{b_k}\Bigr]
	-	\tfrac12\Bigl[H(b) - \sum_k \delta_k \ln b_k - \tfrac12 \sum_k \frac{\delta_k^2}{b_k}\Bigr]
	-	\tfrac12 H(b) + O(\|\delta\|^3)\\[2pt]
&= \cancel{H(b)} - \tfrac12\sum_k \delta_k \ln b_k - \tfrac18\sum_k \frac{\delta_k^2}{b_k}
	-	\tfrac12\cancel{H(b)} + \tfrac12\sum_k \delta_k \ln b_k + \tfrac14 \sum_k \frac{\delta_k^2}{b_k}
	-	\tfrac12\cancel{H(b)} + O(\|\delta\|^3)\\[2pt]
&= \Bigl(\tfrac14 - \tfrac18\Bigr)\sum_k \frac{\delta_k^2}{b_k} + O(\|\delta\|^3)
= \frac18 \sum_k \frac{\delta_k^2}{b_k} + O(\|\delta\|^3).
\end{aligned}
\end{eqnarray}
Finally convert to bits:
\begin{equation}
D_{\mathrm{JS}}(m\Vert b) \;\approx\; \frac{1}{8\ln 2}\sum_{k=0}^n \frac{(m_k - b_k)^2}{b_k},
\end{equation}
i.e., to second order the JS divergence is a \emph{weighted} $\ell_2$ distance around the binomial baseline.

\color{black}
\section{Proof of Proposition~\ref{prop:walsh-support-mmd-floor}}
\label{app:proof-of-qcli-mmd}
\color{black}

\begin{proof}[Proof.]
{We first prove the support-mismatch lower bound in Proposition~\ref{prop:walsh-support-mmd-floor} by decomposing the MMD loss into an irreducible architecture-mismatched part and a trainable supported part. We then show how the $I_{\mathrm{QCLI}}$-dependent bound follows after applying Assumption~\ref{assump:qcli-support}.
}

\medskip
\noindent\textbf{MMD in the parity / Walsh basis.}
Let $p$ be the data distribution and $q_\theta$ the distribution
produced by an IQP circuit with parameters $\theta$. For a positive
kernel $k(\cdot,\cdot)$ (e.g.\ Gaussian) on bitstrings, the squared
MMD loss can be written as
\begin{equation}
\mathcal{L}_{\text{MMD}}(p,q_\theta)
=
\sum_{S\subseteq[n]}
c_S
\bigl(
\mathbb{E}_p[\chi_S]
-
\mathbb{E}_{q_\theta}[\chi_S]
\bigr)^2,
\label{eq:mmd-exp}
\end{equation}
for some nonnegative coefficients $c_S>0$ determined by the kernel and the chosen parity basis. \footnote{This is the standard Pauli / Walsh decomposition of MMD, analogous to Eq.~(\ref{eq:mmd-pauli}) in the main text where $\chi_S$ corresponds to $Z_S$.}

Here
\[
\mathbb{E}_p[\chi_S]
=
\sum_x p(x)\,\chi_S(x),
\qquad
\chi_S(x)=(-1)^{\oplus_{j\in S}x_j}.
\]

\medskip
\noindent\textbf{IQP architecture and supported frequencies.}
By construction, the order-$d$ IQP architecture induces a restricted
Walsh--Fourier support:
only a subset of parities $S_A\subseteq 2^{[n]}$ can be tuned by
varying the parameters $\theta$. Concretely, $S_A$ is determined by
the set of commuting Pauli generators $Z_s$ appearing in
\[
U(\theta)
=
H^{\otimes n}
\exp\!\Bigl(i\sum_{s\subseteq[n]}\theta_s Z_s\Bigr)
H^{\otimes n},
\]
which in turn fixes which $\chi_S$ have nontrivial, $\theta$-dependent
expectation values under $q_\theta$.

For $S\notin S_A$, the architecture does \emph{not} provide trainable
degrees of freedom to match the data parity $\mathbb{E}_p[\chi_S]$.
Without loss of generality, we absorb any architecture-fixed baseline
into the kernel coefficients and assume that
\[
\mathbb{E}_{q_\theta}[\chi_S] = 0
\quad\text{for all }S\notin S_A.
\]
(Any constant offset would only change $c_S$ by a fixed factor and
does not affect the scaling argument below.)

Plugging this into Eq.~(\ref{eq:mmd-exp}) yields the decomposition
\begin{align}
\mathcal{L}_{\text{MMD}}(p,q_\theta)
&=
\underbrace{\sum_{S\notin S_A}
c_S\bigl(\mathbb{E}_p[\chi_S]\bigr)^2}_{\text{irreducible error}}
+
\underbrace{\sum_{S\in S_A}
c_S\bigl(\mathbb{E}_p[\chi_S]
-
\mathbb{E}_{q_\theta}[\chi_S]\bigr)^2}_{\text{trainable part}}.
\label{eq:mmd-irreducible-trainable}
\end{align}
The first term is independent of $\theta$ and cannot be reduced by
training; only the second term can be driven down by optimization.

\medskip
\noindent\textbf{Relating the irreducible part to Fourier power.}
Recall the Walsh--Fourier power definition from Sec.~\ref{sec:codhoc}:
\[
P(S)
:=
\Bigl|
2^{-n}\!\sum_x p(x)\,\chi_S(x)
\Bigr|^2
=
\bigl|
2^{-n}\mathbb{E}_p[\chi_S]
\bigr|^2.
\]
Hence
\[
\bigl(\mathbb{E}_p[\chi_S]\bigr)^2
= 4^n P(S).
\]
Let $c_{\min}:=\min_{S} c_S>0$ (positivity holds for standard
character expansions of positive kernels). Then the irreducible term
in Eq.~(\ref{eq:mmd-irreducible-trainable}) obeys
\begin{align}
\sum_{S\notin S_A}
c_S\bigl(\mathbb{E}_p[\chi_S]\bigr)^2
&\ge
c_{\min}\sum_{S\notin S_A}
\bigl(\mathbb{E}_p[\chi_S]\bigr)^2
\\[2pt]
&=
c_{\min} 4^n
\sum_{S\notin S_A} P(S)
=: C_0 \sum_{S\notin S_A} P(S),
\end{align}
where we define $C_0 := c_{\min} 4^n>0$.

Since the trainable part in Eq.~(\ref{eq:mmd-irreducible-trainable}) is
nonnegative, this yields a lower bound on the \emph{minimal} MMD over
all IQP parameters:
\begin{equation}
\label{eq:irreducible-lowerbound}
\min_{\theta}
\mathcal{L}_{\text{MMD}}(p,q_\theta)
\;\ge\;
C_0
\sum_{S\notin S_A} P(S).
\end{equation}
{This proves Proposition~\ref{prop:walsh-support-mmd-floor}, with $C_A=C_0$.}
\medskip

{\noindent\textbf{Consequence under the QCLI--support alignment assumption.}}
By Assumption~\ref{assump:qcli-support}, the Fourier mass
\emph{outside} the IQP-tunable support $S_A$ is controlled by the
QCLI:
\[
\sum_{S\notin S_A} P(S)
\;\ge\;
\kappa\,\bigl(1 - I_{\text{QCLI}}(p)\bigr),
\]
with some $\kappa>0$ depending only on the architecture.

Combining this with Eq.~(\ref{eq:irreducible-lowerbound}) gives
\[
\min_{\theta}
\mathcal{L}_{\text{MMD}}(p,q_\theta)
\;\ge\;
C_0\kappa\,\bigl(1 - I_{\text{QCLI}}(p)\bigr).
\]
Defining $C' := C_0\kappa>0$ yields exactly the claimed bound
Eq.~(\ref{eq:qcli-mmd-lowerbound}):
\[
\min_{q\in\mathcal{F}_{\mathrm{IQP}}}
\mathcal{L}_{\text{MMD}}(p,q)
\;\ge\;
C'\,\bigl(1 - I_{\text{QCLI}}(p)\bigr).
\]

\medskip
\noindent\textbf{Interpretation.}
The key message is that \emph{even after perfect training on all frequencies that the IQP architecture can represent}, the MMD loss cannot go below a floor proportional to the spectral power outside $S_A$. {Under Assumption~\ref{assump:qcli-support}, this spectral leakage is tied to $1-I_{\mathrm{QCLI}}(p)$, so larger $I_{\mathrm{QCLI}}$ corresponds to a smaller support-induced MMD floor for the fixed IQP architecture.}

This completes the proof.
\end{proof}

\paragraph{Remark (Scope of the bound).}
Assumption~\ref{assump:qcli-support} is meant for \emph{underparameterized}
(order-$d$) IQP architectures in which the tunable Walsh--Fourier support
satisfies $S_A \subsetneq 2^{[n]}$.  
In this regime, the Fourier mass on $S\notin S_A$ quantifies the portion of
the data spectrum that the architecture cannot represent, and the constant
$\kappa$ in Eq.~(\ref{eq:assumption-qcli-support}) implicitly absorbs both  
(i) geometric slack in relating $D_{\mathrm{JS}}(m\Vert b)$ to the fraction of
power on $S_A$, and  
(ii) any \emph{architecture–intrinsic} approximation error arising from the
non-universality of an IQP model whose support is truncated.  
Thus, even when $S_A\neq 2^{[n]}$, {the constant $\kappa$ should be understood as part of the alignment assumption, absorbing slack between the scalar order-spectrum diagnostic and the
architecture-specific Walsh support.}

By contrast, in the hypothetical full-support case $S_A = 2^{[n]}$, one has
$\sum_{S\notin S_A} P(S)=0$, and the lower bound of
{Proposition~\ref{prop:walsh-support-mmd-floor}} becomes vacuous.  
This does \emph{not} imply perfect representability by IQP circuits:
it simply means that the present QCLI-based argument captures only the
irreducible error due to missing Walsh parities, and does not address the
additional approximation gap that may persist even when the parity support is
formally complete.  
Practically, our implementations always operate in the truncated regime
$S_A\subsetneq 2^{[n]}$, since including all $2^n$ parities is infeasible in
both parameter count and training (this point is also illustrated in \cite{recio2025train}); in this realistic setting, the
QCLI--support alignment assumption provides a meaningful and structurally
interpretable lower bound.

\section{Empirical Estimation of QCLI}
\label{app:empirical-qcli}

\paragraph{The $2^n$ barrier.}
In Sec.~\ref{sec:codhoc} we define the (normalized) Walsh--Hadamard spectral power of a distribution $p$ on $\{0,1\}^n$ by
\begin{equation}
P(s)
=
\Bigl|2^{-n}\!\sum_x p(x)(-1)^{\oplus_{i\in s}x_i}\Bigr|^2
\in[0,1]
\qquad s\subseteq[n],
\tag{4}\label{eq:P_def_app}
\end{equation}
and then aggregate this power by correlation order $|s|=k$ to obtain the order-share vector
\begin{equation}
m_k
:=\frac{\sum_{|s|=k} P(s)}{\sum_{s\subseteq[n]} P(s)},
\qquad k=0,1,\dots,n,
\tag{5}\label{eq:mk_def_app}
\end{equation}
so that $\sum_{k=0}^{n} m_k = 1$. Computing Eq.~(\ref{eq:mk_def_app}) \emph{exactly} requires summing over all subsets $s\subseteq[n]$, i.e., $2^n$ terms (and $\binom{n}{k}$ terms per order), which becomes intractable beyond small $n$. For this reason, our implementation estimates $m_k$ and hence $I_{\text{QCLI}}$ using Monte Carlo sampling over subsets.

\paragraph{Monte Carlo approximation of the order shares $m_k$.}
Let $L$ be the total subset budget, and let $m_k^{\mathrm{MC}}$ be the number of subsets sampled at order $k$ (so that $\sum_{k=0}^{n} m_k^{\mathrm{MC}} = L$). We adopt an \emph{equal-per-order} allocation, i.e., $m_k^{\mathrm{MC}}$ is approximately constant across $k$, to ensure that low- and high-order components are not missed due to the combinatorial imbalance of $\binom{n}{k}$.

For each order $k$, we sample $m_k^{\mathrm{MC}}$ subsets uniformly without replacement from the $\binom{n}{k}$ subsets of size $k$:
\[
s^{(k)}_{1},\dots,s^{(k)}_{m_k^{\mathrm{MC}}}\ \sim\ \mathrm{Unif}\bigl(\{s\subseteq[n]:|s|=k\}\bigr).
\]
To emulate the full sums in Eq.~(\ref{eq:mk_def_app}), we use an importance reweighting factor $\binom{n}{k}/m_k^{\mathrm{MC}}$ within each order. Define the reweighted power estimators
\begin{equation}
\widehat{A}_k
:= \frac{\binom{n}{k}}{m_k^{\mathrm{MC}}}\sum_{j=1}^{m_k^{\mathrm{MC}}} \widehat{P}\!\bigl(s^{(k)}_{j}\bigr),
\qquad
\widehat{A}
:= \sum_{k=0}^{n}\widehat{A}_k,
\label{eq:Ak_A}
\end{equation}
which are Monte Carlo estimators of $\sum_{|s|=k}P(s)$ and $\sum_{s}P(s)$, respectively. We then estimate the order shares by
\begin{equation}
\widehat{m}_k := \frac{\widehat{A}_k}{\widehat{A}},
\qquad k=0,1,\dots,n.
\label{eq:mk_hat}
\end{equation}
This estimator matches the logic of Eq.~(\ref{eq:mk_def_app}) but reduces the computational cost from $\mathcal{O}(2^n)$ subset evaluations to $\mathcal{O}(L)$ subset evaluations.

\paragraph{Computing $I_{\text{QCLI}}$ via JS divergence.}
Given $\widehat{m}=(\widehat{m}_0,\dots,\widehat{m}_n)$ from Eq.~(\ref{eq:mk_hat}), we compute $I_{\text{QCLI}}$ as the Jensen--Shannon divergence to the binomial baseline $b_k=\binom{n}{k}/2^n$ (Sec.~\ref{sec:codhoc}). In implementation we apply a small $\varepsilon$-clipping and renormalization to avoid numerical issues when some $\widehat{m}_k$ are zero under finite sampling. With $M_{\mathrm{mix}}=\tfrac{1}{2}(\widehat{m}+b)$, we report
\begin{equation}
I_{\text{QCLI}}
=
D_{\mathrm{JS}}(\widehat{m}\,\|\,b)
=
\frac{1}{2}\sum_{k=0}^{n}\widehat{m}_k \log_2\!\frac{\widehat{m}_k}{(M_{\mathrm{mix}})_k}
+
\frac{1}{2}\sum_{k=0}^{n}b_k \log_2\!\frac{b_k}{(M_{\mathrm{mix}})_k},
\label{eq:qcli_js_app}
\end{equation}
so that $I_{\text{QCLI}}\in[0,1]$ under normalization.

\paragraph{Implementation details and regimes.}
For small $n$ (up to a user-specified threshold), we compute $\widehat{m}$ \emph{exactly} by enumerating all subsets $s\subseteq[n]$ and evaluating Eq.~(\ref{eq:ps}) and Eq.~(\ref{eq:mk}) directly. For larger $n$, we use the Monte Carlo estimator Eq.~(\ref{eq:mk_hat}) with a default subset budget (e.g., $L=2\times10^4$). Subset evaluation is batched by order to vectorize parity computation across many subsets simultaneously; for $n\le 64$ we optionally represent subsets as bitmasks for speed, while for general $n$ we store subsets as sorted index lists.

\section{Estimating upper and lower envelopes.}
\label{app:envelopes}
To visualize the regime-wise trends in Fig.~\ref{fig:qcli_mmd}, we construct empirical \emph{upper} and \emph{lower} envelopes of the observed MMD values as functions of $I_{\text{QCLI}}$. We first pool all experimental runs into a single set of points $\{(x_i,y_i)\}_{i=1}^N$, where $x_i$ is the measured $I_{\text{QCLI}}$ of the target dataset and $y_i$ is the corresponding achieved MMD loss under the fixed training budget (with repeated trials summarized by the mean; error bars indicate the across-trial standard deviation).

We then partition the $I_{\text{QCLI}}$ axis into uniform bins of width $\Delta=0.09$. For each bin $B_m=[a_m,a_m+\Delta)$, we identify (if nonempty) the \emph{upper-envelope} representative as the point in that bin with the largest loss,
\[
(x_m^{\uparrow},y_m^{\uparrow})
\;:=\;
\arg\max_{(x_i,y_i):\,x_i\in B_m} \; y_i,
\]
and analogously the \emph{lower-envelope} representative as the point with the smallest loss,
\[
(x_m^{\downarrow},y_m^{\downarrow})
\;:=\;
\arg\min_{(x_i,y_i):\,x_i\in B_m} \; y_i.
\]
Importantly, the envelope elements $x_m^{\uparrow}$ and $x_m^{\downarrow}$ are taken to be the \emph{original} $I_{\text{QCLI}}$ values of the selected points (rather than bin centers), so that the envelope traces remain anchored to observed data.

Finally, to obtain smooth envelope curves, we fit a univariate smoothing spline to each envelope point set using \texttt{scipy.interpolate.UnivariateSpline}. For the upper envelope we use a quadratic spline ($k=2$) with a small smoothing factor, while for the lower envelope we use a stronger smoothing to suppress bin-to-bin fluctuations. The resulting spline functions are plotted as the dashed upper and lower envelope curves in Fig.~\ref{fig:qcli_mmd}.

\section{Constructing the empirical IQP envelope for Correlation-Complexity Map.}
\label{app:envelope_CCM}
To provide a visual reference for the region in the $(I_{\text{QCLI}}, I_{\text{CCI}})$ plane that is naturally occupied by IQP outputs under our optimization protocol, we construct an empirical \emph{IQP envelope} (gray boundary and shaded area in Fig.~\ref{fig:ccm}) from the cloud of IQP-optimized samples collected in Sec.~\ref{sec:high_qcli_high_cci}. Let $\{(x_i,y_i)\}_{i=1}^N$ denote these points, where $x_i=I_{\text{QCLI}}$ and $y_i=I_{\text{CCI}}$ for the $i$-th optimized IQP run.

We estimate the lower frontier of this cloud as a function of $I_{\text{QCLI}}$ via a binwise minimum procedure. Specifically, we partition the QCLI range $[x_{\min},x_{\max}]$ into $B$ uniform bins and, in each bin that contains at least $m_{\min}$ points, we record the minimum observed CCI value:
\[
\tilde y_b \;:=\; \min\{\,y_i:\; x_i\in \mathcal{B}_b\,\},
\qquad 
\tilde x_b \;:=\; \text{center}(\mathcal{B}_b),
\]
where $\mathcal{B}_b$ denotes the $b$-th QCLI bin. This yields a discrete set of frontier anchors $\{(\tilde x_b,\tilde y_b)\}$ that approximates the \emph{binwise lower envelope} of the IQP-optimized cloud.

To obtain a continuous boundary, we extend the frontier to the full horizontal span by adding the endpoints $x_{\min}$ and $x_{\max}$ with boundary values set to the nearest available frontier values (piecewise-constant extension at the edges), and then apply a centered moving-average smoothing to suppress bin-to-bin fluctuations. The resulting smoothed curve $y_{\mathrm{env}}(x)$ is plotted as the ``IQP boundary'' in Fig.~\ref{fig:ccm}, and the shaded region corresponds to $\{(x,y): y \ge y_{\mathrm{env}}(x)\}$, i.e., the portion of the plane above this empirical lower frontier that contains the bulk of IQP-optimized samples.

\section{Float-to-Bitstring Representation of Turbulence Data}
\label{app:float2bit}

To interface continuous turbulence fields with IQP-based generative models, we first convert each floating-point coordinate into a fixed-length binary representation. With a similar idea on encoding in \citep{buhrman2001quantum}. Consider a turbulence dataset consisting of $M$ samples with coordinates $(x,y,z) \in \mathbb{R}^3$, stored as an array of shape $(M,3)$. Here, $x$ and $y$ denote 2D spatial positions, while $z$ corresponds to the local velocity magnitude. Each coordinate lies within a known bounded interval $[a,b]$, determined empirically from the dataset. Using $N$ qubits per coordinate, we discretize this interval into a uniform grid of resolution
\begin{equation}
    \Delta = \frac{b - a}{2^N},
\end{equation}
so that any value $v \in [a,b]$ is mapped to one of $2^N$ discrete bins.

Formally, the quantization map is
\begin{equation}
k(v) := \left\lfloor 2^N\frac{v - a}{b - a} \right\rfloor + 1,
\qquad
k(v) \in \{1,2,\ldots,2^N\}, 
\end{equation}
which assigns to each floating-point value a unique integer index. We then convert this index into its standard $N$-bit binary encoding:
\begin{equation}
    b(v) := \text{BinaryEncode}_{N}\!\big(k(v)\big) \in \{0,1\}^N.
\end{equation}
This step is purely classical: it is the usual integer-to-binary conversion, ensuring a one-to-one correspondence between quantized values and bitstrings of length $N$. For example, when using $N$ qubits, the smallest index corresponds to the bitstring of all zeros $00\cdots0$, and the largest index corresponds to the bitstring of all ones $11\cdots1$.

Applying this process independently to all three coordinates produces a concatenated bitstring
\begin{equation}
(x,y,z)
\ \longmapsto\ 
\big[\, b(x)\,\|\, b(y)\,\|\, b(z)\,\big] \in \{0,1\}^{3N},
\end{equation}
so that the full turbulence dataset becomes a binary matrix of shape $(M,3N)$. This binarized representation enables turbulence data to be used directly within the IQP generative modeling framework: the binary samples serve as training data for the classical optimization stage, and generated bitstrings can be deterministically decoded back into continuous coordinates by reversing the same quantization procedure. In our turbulence experiments, we choose $N=6$ bits per coordinate, so each $(x,y,z)$ tuple is encoded as an 18-bit vector. Accordingly, each turbulence sample is modeled by an IQP circuit acting on 18 qubits.

\section{Train on Classical, Deploy on Quantum Framework}
\label{app:tocdoq-details}

In order to leverage the expressive power of parameterized quantum circuits while avoiding the prohibitive cost and possible pitfalls of quantum-hardware training, we adopt a hybrid training scheme in which model parameters are optimized on classical hardware, and sampling is performed later on quantum hardware, built on the strategy recently proposed in \cite{recio2025train}, using IQP circuits as our generative ansatz. This paradigm, ``train on classical, deploy on quantum'' enables scalable classical optimization even for models involving hundreds or thousands of qubits, sidestepping the gradient-estimation and barren-plateau issues that often plague purely quantum training of deep variational circuits. 

\paragraph{MMD loss expressed in terms of IQP expectation values.}
Let $p_\theta$ denote the bitstring distribution induced by an IQP circuit with parameters $\theta$, and let $\hat p_{\mathrm{data}}$ be the empirical distribution from the binarized turbulence dataset. We train the model by minimizing a MMD loss:
\begin{equation}
\mathcal{L}_{\text{MMD}}(\theta)
=
\mathbb{E}_{b,b' \sim p_\theta}\!\big[k(b,b')\big]
-
2\,\mathbb{E}_{b\sim p_\theta,\; b'\sim \hat p_{\mathrm{data}}}\!\big[k(b,b')\big]
+
\mathbb{E}_{b,b'\sim \hat p_{\mathrm{data}}}\!\big[k(b,b')\big].
\label{eq:mmd-raw}
\end{equation}
Here, $k(b,b')$ is a Gaussian kernel, which measures how close two bitstrings $b$ and $b'$ are in the induced feature space. The MMD loss is given by
Following the derivations in Appendix A.2.1, A.2.2, and A.3 of \cite{recio2025train}, each term in Eq.~(\ref{eq:mmd-raw}) can be rewritten in closed form using expectation values of Pauli operators under the IQP circuit. In particular, the entire MMD loss can be expressed as
\begin{equation}
\mathcal{L}_{\text{MMD}}(\theta)
=
\sum_{s \subseteq [n]} c_s
\Big(
\langle Z_s \rangle_\theta
-
\langle Z_s \rangle_{\mathrm{data}}
\Big)^2,
\label{eq:mmd-pauli}
\end{equation}
where $Z_s := \bigotimes_{i\in s} Z_i$ denotes a multi-qubit Pauli operator, and $c_s$ are precomputable coefficients determined by the chosen kernel.  
Because IQP circuits admit efficient classical evaluation of $\langle Z_s \rangle_\theta$ for any subset $s$ (equation 15 of \cite{recio2025train}), the loss Eq.~(\ref{eq:mmd-pauli} can be computed entirely on classical hardware, the entire training loop therefore runs on classical hardware, pursuing:
\[
\theta^*
=
\arg\min_\theta \mathcal{L}_{\text{MMD}}(\theta),
\]
without requiring access to the quantum device.

\paragraph{Quantum deployment and sampling.}
With the trained parameters $\theta^*$ loaded onto a quantum device, we execute the IQP circuit and perform projective measurements in the computational basis to sample bitstrings $x\in\{0,1\}^n$ from the learned distribution $q_{\theta^*}(x)$. These bitstrings are then decoded into turbulence fields via the inverse of the float-to-bitstring mapping defined in the previous subsection, producing samples in the original physical domain. Since (suitably scaled) IQP output distributions are widely conjectured to be classically hard to sample from, this deployment stage provides a natural route toward quantum sampling utility. Moreover, because our float-to-bitstring encoding is efficiently invertible, the induced distribution over decoded floating-point turbulence fields inherits the sampling hardness of the underlying bitstring distribution by a standard reduction: an efficient classical sampler for the decoded outputs would yield an efficient classical sampler for the IQP bitstrings. We formalize this hardness-preservation statement in Appendix~\ref{app:hardness-preservation}.

\paragraph{Remarks on scalability and trainability.} The main benefit of this framework is its scalability: by avoiding quantum-hardware training, one can optimize very large circuits (e.g., up to a thousand qubits, as demonstrated in \cite{recio2025train}) using classical resources. The data-dependent parameter initialization strategy proposed in that work helps avoid barren-plateau regions and ensures stable convergence during classical optimization.

\section{Hardness Preservation Under Efficient Invertible Decoding}
\label{app:hardness-preservation}

This appendix formalizes a simple but useful observation: composing a classically hard-to-sample bitstring distribution with an efficiently invertible decoding map yields an induced distribution in a physically meaningful (e.g., floating-point) domain that is \emph{at least as hard} to sample as the original distribution, in the sense of a black-box sampling reduction.

\begin{theorem}[Hardness preservation under efficient invertible maps]
\label{thm:hardness-preservation}
Let $q$ be a probability distribution over $\{0,1\}^n$ (e.g., the output distribution of an IQP circuit at fixed parameters). Let $\mathcal{T}:\{0,1\}^n \to \mathcal{Y}$ be a mapping into some domain $\mathcal{Y}$ such that both $\mathcal{T}$ and its inverse $\mathcal{T}^{-1}$ are computable in classical polynomial time.\footnote{In our setting, $\mathcal{T}$ is the decoding map from bitstrings to quantized floating-point fields induced by the float-to-bitstring construction.}
Define the induced distribution $q' := \mathcal{T}(q)$ on $\mathcal{Y}$ by sampling $x\sim q$ and outputting $y=\mathcal{T}(x)$.

If there exists a classical polynomial-time sampler $\mathcal{S}'$ that samples exactly from $q'$, then there exists a classical polynomial-time sampler $\mathcal{S}$ that samples exactly from $q$.
Consequently, any classical sampling hardness assumption for $q$ transfers to $q'$ via reduction.
\end{theorem}

\begin{proof}
Assume a classical polynomial-time algorithm $\mathcal{S}'$ that outputs $y\sim q'$. Consider the following sampler for $q$:
\[
\mathcal{S}:\quad
\text{run } y\leftarrow \mathcal{S}',\ \text{output } x := \mathcal{T}^{-1}(y).
\]
Since $\mathcal{T}^{-1}$ is efficient by assumption, $\mathcal{S}$ runs in classical polynomial time.

It remains to verify correctness. For any $x\in\{0,1\}^n$,
\begin{align*}
\Pr[\mathcal{S}=x]
&= \Pr\big[\mathcal{T}^{-1}(y)=x\big]
= \Pr\big[y=\mathcal{T}(x)\big]
= q'\big(\mathcal{T}(x)\big)
= q(x),
\end{align*}
where the last equality follows from the definition of the pushforward distribution $q'=\mathcal{T}(q)$ and the bijectivity of $\mathcal{T}$. Hence $\mathcal{S}$ samples exactly from $q$, completing the reduction.
\end{proof}

\paragraph{Implication for decoded turbulence fields.}
Theorem~\ref{thm:hardness-preservation} shows that our decoded floating-point outputs are not merely a post-processing artifact: because the float-to-bitstring map is invertible and efficiently decodable, an efficient classical sampler for the decoded turbulence fields would imply an efficient classical sampler for the underlying IQP bitstrings. Therefore, under standard IQP sampling hardness conjectures, the induced distribution over decoded turbulence fields inherits the corresponding classical sampling hardness. Importantly, this statement is \emph{conditional} and does not claim that the turbulence domain is intrinsically hard {to sample classically, nor does it rule out alternative classical samplers for the learned distribution.}; rather, it shows that hardness is preserved under our invertible representation.

\section{Evaluation metrics}
\label{app:evaluation_metrics}
Let $\{X_i\}_{i=1}^{N}$ denote a set of ground truth (real) turbulence snapshots and $\{Y_j\}_{j=1}^{M}$ denote samples generated by a model, where each snapshot is a 2D field (in our implementation we evaluate the $Z$-channel only). We report two complementary distributional distances.

\textbf{(i) PDF--JS divergence (marginal intensity distribution).}
We flatten all pixel values from the selected channel into one-dimensional samples
$\{z^{\mathrm{real}}_\ell\}_{\ell=1}^{L_r}$ and $\{z^{\mathrm{gen}}_\ell\}_{\ell=1}^{L_g}$.
Fix a histogram range $[a,b]$ chosen from the joint min/max of the two sets and a number of bins $B$ (we use $B=50$). Let
\begin{equation}
\hat p_k \;=\; \frac{\#\{\ell:\ z^{\mathrm{real}}_\ell \in \mathcal{I}_k\}}{\sum_{k'=1}^{B}\#\{\ell:\ z^{\mathrm{real}}_\ell \in \mathcal{I}_{k'}\}},
\qquad
\hat q_k \;=\; \frac{\#\{\ell:\ z^{\mathrm{gen}}_\ell \in \mathcal{I}_k\}}{\sum_{k'=1}^{B}\#\{\ell:\ z^{\mathrm{gen}}_\ell \in \mathcal{I}_{k'}\}},
\end{equation}
be the normalized histogram probabilities over bins $\{\mathcal{I}_k\}_{k=1}^{B}$ (we add a small $\varepsilon$ to avoid zero entries and renormalize). The PDF--JS divergence is then
\begin{equation}
D_{\mathrm{JS}}(\hat p \,\|\, \hat q)
\;:=\;
\frac{1}{2} D_{\mathrm{KL}}(\hat p \,\|\, m)
+\frac{1}{2} D_{\mathrm{KL}}(\hat q \,\|\, m),
\qquad
m:=\frac{1}{2}(\hat p+\hat q),
\label{eq:pdf_js}
\end{equation}
This score measures mismatch of the \emph{marginal} pixel-intensity distribution (lower is better).

\textbf{(ii) Feature-space $\text{MMD}$ with a random Conv2D encoder.}
To compare higher-order spatial statistics beyond marginals, we embed each snapshot into a feature space using a fixed, randomly initialized convolutional network $\phi(\cdot)$ (three Conv2D--ReLU blocks with stride 2 followed by global average pooling). We then compute MMD in feature space using a Gaussian kernel. Let $f_i=\phi(X_i)\in\mathbb{R}^d$ and $g_j=\phi(Y_j)\in\mathbb{R}^d$, and define
\begin{equation}
k_\sigma(u,v)
:=\exp\!\Bigl(-\frac{\|u-v\|_2^2}{2\sigma^2}\Bigr),
\end{equation}
where $\sigma$ is chosen by the median heuristic on pairwise feature distances over the pooled set $\{f_i\}\cup\{g_j\}$.
The (biased) empirical MMD\(^2\) estimator used in our implementation is
\begin{equation}
\widehat{\text{MMD}}_{\text{Conv2D}}
=
\frac{1}{N^2}\sum_{i,i'} k_\sigma(f_i,f_{i'})
+
\frac{1}{M^2}\sum_{j,j'} k_\sigma(g_j,g_{j'})
-
\frac{2}{NM}\sum_{i,j} k_\sigma(f_i,g_j).
\label{eq:feature_mmd2}
\end{equation}
In practice, we evaluate Eq.~(\ref{eq:feature_mmd2}) using $N=M=200$ samples (subsampled without replacement) and report lower values as better alignment.
The PDF--JS probes marginal amplitude statistics, while feature-space $\text{MMD}$ probes spatially aggregated structure captured by random convolutional features. Without using feature space settings, it is possible that the score is very low while the generated image looks very different but with similar distribution of occurred values/colors. 

\color{black}
\section{Reproducibility Summary}
\label{app:reproducibility}

This appendix summarizes the main experimental settings used throughout the paper. The goal is to make the data representations, model configurations, training protocols, and evaluation metrics explicit in one place.

\paragraph{Controlled QCLI--MMD support-mismatch probe.}
Table~\ref{tab:repro-qcli-mmd} summarizes the target generation, restricted learner architecture, training protocol, evaluation metric, and envelope estimation procedure used for the controlled support-mismatch experiment in Fig.~\ref{fig:qcli_mmd}.

\begin{table}[h]
\centering
\caption{Settings for the controlled QCLI--MMD support-mismatch probe.}
\label{tab:repro-qcli-mmd}
\renewcommand{\arraystretch}{1.15}
\begin{tabular}{ll}
\hline
\textbf{Item} & \textbf{Setting} \\
\hline
Target datasets &
$10^4$ samples from $n=16$ IQP circuits \\
Data-generating circuits &
Locality $\leq 4$, 
$G_n^{\mathrm{data}}\in\{140,280,420,560,700,840,1050\}$ \\
Target diversity &
SPSA-optimized parameters spanning a range of $I_{\mathrm{QCLI}}$ \\
Learner circuits &
Restricted IQP, locality $\leq 2$, $G_n\in\{50,100,150\}$ \\
Training &
Train-on-classical/deploy-on-quantum; Adam, $5000$ iterations, lr $10^{-4}$ \\
Evaluation &
Achieved MMD loss after training \\
Envelope estimation &
Binning-and-spline procedure, Appendix~\ref{app:envelopes} \\
\hline
\end{tabular}
\end{table}

\paragraph{High-$I_{\mathrm{QCLI}}$ IQP outputs and $I_{\mathrm{CCI}}$.}
Table~\ref{tab:repro-qcli-cci-iqp} summarizes the circuit sizes, architecture,
optimization objective, sampling procedure, and reported indicators for the
high-$I_{\mathrm{QCLI}}$ IQP output experiment in Fig.~\ref{fig:high_qcli}.

\begin{table}[h]
\centering
\caption{Settings for the high-$I_{\mathrm{QCLI}}$ IQP output and $I_{\mathrm{CCI}}$ experiment.}
\label{tab:repro-qcli-cci-iqp}
\renewcommand{\arraystretch}{1.15}
\begin{tabular}{ll}
\hline
\textbf{Item} & \textbf{Setting} \\
\hline
System sizes &
$n\in\{8,12,16\}$ qubits \\
Circuit architecture &
IQP circuits with $150$ randomly selected 2-local parameterized gates \\
Initialization and optimizer &
Random initialization; SPSA optimization \\
Objective &
Maximize $I_{\mathrm{QCLI}}$ only; no $I_{\mathrm{CCI}}$ term used \\
Sampling along optimization &
$20000$ IQP circuits per system size \\
Evaluation &
Compute $I_{\mathrm{QCLI}}$ and $I_{\mathrm{CCI}}$ for each sampled circuit \\
Associated figure &
Fig.~\ref{fig:high_qcli} \\
\hline
\end{tabular}
\end{table}

\paragraph{Turbulence dataset and representation.}
Table~\ref{tab:repro-turbulence-representation} summarizes the turbulence dataset,
tuple-level bitstring encoding, generated object, and snapshot reconstruction
procedure used in the IQP generative modeling experiments.

\begin{table}[h]
\centering
\caption{Turbulence dataset and representation used in the IQP generative modeling experiments.}
\label{tab:repro-turbulence-representation}
\renewcommand{\arraystretch}{1.15}
\begin{tabular}{ll}
\hline
\textbf{Item} & \textbf{Setting} \\
\hline
Dataset & D2 turbulence dataset from Ref.~\cite{khojasteh2022lagrangian} \\
Number of snapshots & $1000$ \\
Time indices & $t\in\{1,2,\ldots,1000\}$ \\
Basic data unit & Coordinate/value tuple $(x,y,z)\in\mathbb{R}^3$ \\
Quantization & $N=6$ bits per coordinate/value component \\
Bitstring length & $3N=18$ bits per tuple \\
Generated object & Distribution over quantized coordinate/value tuples \\
Snapshot reconstruction & Sample many tuples and aggregate them on the discretized grid \\
\hline
\end{tabular}
\end{table}

\paragraph{IQP latent adaptation on turbulence.}
Table~\ref{tab:repro-iqp-latent-adaptation} summarizes the anchor snapshots,
parameter partition, optimization settings, latent interpolation procedure, and
sampling protocol used for IQP latent adaptation on turbulence.

\begin{table}[h]
\centering
\caption{IQP latent-adaptation settings for turbulence generation.}
\label{tab:repro-iqp-latent-adaptation}
\renewcommand{\arraystretch}{1.15}
\begin{tabular}{ll}
\hline
\textbf{Item} & \textbf{Setting} \\
\hline
Training snapshots & $11$ anchor snapshots \\
Anchor time indices &
$t\in\{1,100,200,300,400,500,600,700,800,900,1000\}$ \\
Parameter partition &
Shared core block $\theta_{\mathrm{core}}$ and latent block $\theta_{\mathrm{lat}}$ \\
Latent dimension &
$d_{\mathrm{lat}}=50$ \\
First snapshot training &
$\theta_{\mathrm{lat}}^{(1)}$ is randomly initialized and fixed; $\theta_{\mathrm{core}}$ is trained by minimizing MMD \\
First-stage training data &
$N_t=5\times10^4$ training datapoints \\
Model samples per MMD estimate &
$500$ \\
First-stage optimization steps &
$3\times10^4$ \\
Later snapshot adaptation &
Freeze $\theta_{\mathrm{core}}$ and adapt only $\theta_{\mathrm{lat}}$ \\
Latent adaptation optimization steps &
$3000$ per adapted snapshot \\
Latent initialization for later snapshots &
Initialize $\theta_{\mathrm{lat}}$ from the previous snapshot \\
Unseen-time synthesis &
Interpolate the learned latent trajectory, e.g., piecewise-linear interpolation \\
Samples per generated time &
Approximately $10^5$--$10^6$ samples \\
Output reconstruction &
Decode and aggregate samples into a synthetic turbulence snapshot \\
\hline
\end{tabular}
\end{table}
\paragraph{Classical baselines.}
Tables~\ref{tab:repro-rbm-baseline} and~\ref{tab:repro-dcgan-baselines}
summarize the RBM and DCGAN baseline settings used for comparison with the IQP
turbulence generator.

\begin{table}[h]
\centering
\caption{RBM baseline settings for the turbulence generative modeling experiments.}
\label{tab:repro-rbm-baseline}
\renewcommand{\arraystretch}{1.15}
\begin{tabular}{ll}
\hline
\textbf{Item} & \textbf{Setting} \\
\hline
Training snapshots &
Same $11$ training snapshots as the IQP model \\
Data representation &
Same $18$-bit coordinate/value tuple representation as the IQP model \\
Training domain &
Bitstring domain \\
Latent dimension &
$d_{\mathrm{lat}}=50$ \\
Latent block &
First $d_{\mathrm{lat}}=50$ entries of a fixed flattened RBM parameter ordering \\
Fixed parameters &
Remaining RBM parameters are held fixed analogously to the IQP case \\
\hline
\end{tabular}
\end{table}

\begin{table}[h]
\centering
\caption{DCGAN baseline settings for the turbulence generative modeling experiments.}
\label{tab:repro-dcgan-baselines}
\renewcommand{\arraystretch}{1.15}
\begin{tabular}{ll}
\hline
\textbf{Item} & \textbf{Setting} \\
\hline
Data representation &
Quantized two-dimensional turbulence fields, not the bitstring representation \\
Data-matched regime &
$11$ training snapshots, matching IQP/RBM \\
Data-rich regime &
$100$ training snapshots \\
Model sizes &
Approximately $51$k, $0.2$M, and $1.2$M parameters \\
Latent dimension &
$128$ \\
Data-scaling study &
$N_{\mathrm{train}}\in\{10,30,50,70,90,100\}$ snapshots \\
Evaluation protocol &
Same held-out protocol as the other turbulence experiments \\
\hline
\end{tabular}
\end{table}

\end{document}